\documentclass{article}

\PassOptionsToPackage{numbers, compress}{natbib}
\usepackage[preprint]{neurips_2026}

\usepackage[utf8]{inputenc}
\usepackage[T1]{fontenc}
\usepackage[hidelinks]{hyperref}
\usepackage{url}
\usepackage{booktabs}
\usepackage{amsmath,amssymb,amsthm,mathtools,bm}
\usepackage{array}
\usepackage{graphicx}
\usepackage{nicefrac}
\usepackage{microtype}
\usepackage{xcolor}
\definecolor{codepink}{RGB}{185,45,100}
\usepackage{enumitem}

\newtheorem{theorem}{Theorem}
\newtheorem{proposition}[theorem]{Proposition}
\newtheorem{corollary}[theorem]{Corollary}

\newtheorem{remark}[theorem]{Remark}

\newcommand{\R}{\mathbb{R}}
\newcommand{\E}{\mathbb{E}}
\newcommand{\cL}{\mathcal{L}}
\newcommand{\cD}{\mathcal{D}}
\newcommand{\cU}{\mathcal{U}}
\newcommand{\cT}{\mathcal{T}}
\newcommand{\Proj}{P_{\cU}}
\newcommand{\ip}[2]{\left\langle #1,#2\right\rangle}
\newcommand{\norm}[1]{\left\lVert #1\right\rVert}

\title{When Do Surrogate Updates Improve Decisions? A Local Theory of Trajectory-Wise Transfer}
\author{%
Yuyang Shen\\
The Chinese University of Hong Kong, Shenzhen\\
\texttt{yuyangshen@link.cuhk.edu.cn}\\
\\
{\small Code:
\href{https://github.com/Ethan-Shen-Individual-Lab/surrogate-to-decision-transfer}
{\textcolor{codepink}{\nolinkurl{surrogate-to-decision-transfer}}}}}

\begin{document}
\maketitle

\renewcommand{\footnoterule}{}
\begingroup
\renewcommand{\thefootnote}{}
\footnotetext{\phantom{Corresponding author.}}
\endgroup

\begin{abstract}
A broad range of models face the mismatch where they are updated through trajectory losses but are evaluated by downstream task reward. Here, a trajectory is a training instance that induces a surrogate loss whose reduction might not track the model's decision utility update. Theoretically, we ask when one step of trajectory training reduces both population surrogate loss and decision risk, and how transfer accumulates along repeated updates. To formalize this, we first fix a checkpoint and a restricted update space, and define the reductions in population surrogate risk and decision risk induced by a trajectory as its \emph{learnability} and \emph{decision utility}, respectively. On this basis, our theory yields four main results. First, a one-step transfer bound separates their discrepancy into first-order gradient misalignment after nonnegative calibration and second-order curvature; and a pathwise extension accumulates the same terms over repeated updates. Second, when the accessible surrogate gradient is nonzero, universal first-order transfer over every accessible direction holds exactly when the accessible surrogate and decision gradients are positively collinear. Third, the calibration gap bounds the decision regret of learnability-based trajectory selection, while a candidate-difference refinement tightens this guarantee by retaining only directions that affect pairwise rankings. Finally, we establish an approximation--calibration trade-off across nested update spaces. Controlled gridworld experiments recover the predicted quadratic Taylor scaling, with cumulative discrepancies remaining below the plug-in envelope. Across RL and LLM post-training, controlled misalignment degrades candidate rankings; in the LLM setting, corruption shifts the gradient cosine from $0.074$ to $-0.605$ and yields decision-harmful cross-fitted witnesses in all five seeds. Nested-space experiments exhibit the predicted approximation--calibration trade-off.
\end{abstract}

\section{Introduction}
\label{sec:introduction}
Across diverse model-training paradigms, including imitation learning \citep{ross2011reduction}, policy-gradient and actor--critic reinforcement learning \citep{sutton1999policy,konda1999actor}, language-model post-training \citep{ouyang2022training,rafailov2023direct}, and test-time adaptation \citep{wang2021tent}, models are updated using trajectories or small candidate pools. Here, a trajectory is a training instance, which may be a rollout, demonstration, rationale--answer example, or preference-labeled response. Each trajectory induces a surrogate objective, such as a demonstration-based supervised loss, a reward- or advantage-weighted loss, or a preference-based loss \citep{ross2011reduction,schulman2017ppo,magister2023teaching,rafailov2023direct}. Such differentiable signals might not track action return, answer correctness, or other downstream decision task rewards \citep{bartlett2006convexity,liu2021risk,elmachtoub2022smart}. This mismatch matters when choosing trajectories \citep{mindermann2022prioritized,xia2024less} or trainable parameters \citep{hu2022lora,wang2021tent} during model training, since such choices require knowing whether the training signal at the current checkpoint is informative about the downstream effect.

We study this problem from a local level perspective. At a fixed iterate, the trainable head, adapter, or model defines an accessible subspace; trajectory $\tau$ induces gradient $g_\tau$ and displacement $-\eta g_\tau$. First-order transfer compares surrogate risk and decision risk under the same change through the alignment of their accessible population gradients up to a nonnegative scale. It is iterate- and space-specific, and expanding the trainable set can improve attainable decisions while exposing misaligned directions.

We derive one-step and pathwise transfer bounds, characterize universal first-order transfer by positive collinearity, and bound learnability-based selection regret using calibration, curvature, and bounded candidate gradients. For finite pools, we tighten this bound to the span of candidate-gradient differences, discarding residual directions invisible to pairwise rankings. We also expose an approximation--calibration trade-off across nested update spaces. These diagnostics inform supervision, selection, and trainable-parameter design without changing the objective.

Unlike global surrogate--regret analyses of excess risk \citep{bartlett2006convexity,tewari2007consistency,osokin2017structured} or decision-focused methods that modify the objective \citep{donti2017task,elmachtoub2022smart,mandi2024decision}, we ask whether an existing update direction is locally decision-calibrated, how its error accumulates pathwise, and what this implies for trajectory selection.

\paragraph{Contributions.} Our main contributions are as follows:
\begin{itemize}[leftmargin=*]
    \item \textbf{Trajectory-wise local transfer.} Taking a trajectory-wise perspective, we evaluate each trajectory-induced parameter update under two population objectives: its reduction in surrogate risk, termed \emph{learnability}, and its reduction in decision risk, termed \emph{decision utility}. This formulation makes their local relationship explicit within the update space accessible to the chosen trainable parameters.
    \item \textbf{Transfer guarantees and exact boundary.} We derive one-step and cumulative bounds that separate first-order calibration error from second-order curvature effects, and, for a nonzero accessible surrogate gradient, prove that universal first-order transfer holds exactly under positive collinearity.
    \item \textbf{Trajectory selection.}
    For a finite pool of candidate trajectories with bounded gradients, we bound the one-step decision-to-utility regret of learnability-based selection in terms of the local calibration gap and curvature. We further sharpen it by measuring misalignment only along pairwise candidate-gradient differences, thereby removing directions that cannot affect candidate rankings. 
    \item \textbf{Update-space design.} For nested update spaces, we establish approximation--calibration tradeoff: enlarging the trainable subspace decreases approximation error while increasing calibration gap.
    \item \textbf{Controlled validation.} Experiments with a PPO-initialized gridworld policy and LLM post-training test the transfer relation, its failure under misalignment, and the predicted selection and update-space effects, providing diagnostics for policy learning and language-model adaptation.
\end{itemize}

\section{Related Work}
\label{sec:related-work}

\paragraph{Surrogate calibration and regret transfer.}
Classical learning theory asks when minimizing a tractable surrogate is consistent for a target loss, and relates their excess risks through calibration functions \citep{bartlett2006convexity,tewari2007consistency}. Extensions cover structured prediction, where targets and surrogates may be multivariate \citep{osokin2017structured}. These results compare population risks relative to their optima under global, distribution-level conditions. We retain their direction and finite step size, studying one-step and pathwise transfer through a local geometric gap distinct from probabilistic confidence calibration.

\paragraph{Decision-focused learning and downstream-aware selection.}
Decision-focused and predict-then-optimize methods train predictors through downstream optimization objectives \citep{donti2017task,elmachtoub2022smart,mandi2024decision}. Related work develops calibrated decision-aware surrogates \citep{liu2021risk} and directional-gradient losses \citep{gupta2024directional}. For LLM post-training, NICE uses policy-gradient signals to select examples by downstream utility \citep{wang2025nice}. These approaches modify the objective or selection rule to incorporate downstream information. 

\paragraph{Trajectory-based learning and training-time choices.}
Imitation learning and policy optimization obtain tractable updates from demonstrations or reward-weighted trajectories \citep{ross2011reduction,schulman2017ppo}. Curriculum and data-selection methods rank examples by difficulty, training progress, or learnability \citep{bengio2009curriculum,mindermann2022prioritized}. Gradient-based variants match training or validation gradients \citep{killamsetty2021gradmatch}, mitigate conflicts among objectives \citep{yu2020gradient}, or select LLM tuning data through low-rank gradient similarity \citep{xia2024less}. A separate choice is which parameters may move: LoRA imposes a low-rank adaptation parameterization \citep{hu2022lora}, while TENT updates restricted normalization parameters at test time \citep{wang2021tent}. Comparatively, our framework places trajectory and trainable-parameter choices in a common accessible geometry. It bounds learnability-based selection regret, sharpens finite-pool control to candidate-gradient differences, and separates the approximation benefit of a larger update space from its potential calibration cost.

\section{Problem Setting}
\label{sec:problem-setting}

\paragraph{Policy and decision risk.}
Let $\pi_\theta$ be a stochastic policy with parameter vector $\theta\in\R^p$. The input $X$ may be features, a prompt, state, or decision history, while $A\sim\pi_\theta(\cdot\mid X)$ may be a prediction, response, or policy-controlled actions in a rollout. For a bounded task reward $r(X,A)$, we define
\begin{equation}
    J(\theta)=\E[r(X,A)],
    \qquad
    \cD(\theta)=-J(\theta),
    \label{eq:decision-risk}
\end{equation}
where the expectation is over the evaluation distribution and policy. Thus, minimizing $\cD$ equals to maximize expected reward. The optimal decision risk within the parameterized policy class is
\[
    \cD^*=\inf_{\vartheta\in\R^p}\cD(\vartheta).
\]

\paragraph{Restricted update space.}
Let $\theta_0$ be a reference checkpoint, and $\cU\subseteq\R^p$ be a linear update subspace with orthogonal projector $\Proj$. Therefore, the accessible parameter class is
\begin{equation}
    \Theta_{\cU}=\theta_0+\cU.
\end{equation}
This representation covers frozen-coordinate and frozen-layer adaptation exactly and locally models parameter-efficient methods through their tangent spaces \citep{wang2021tent,hu2022lora}. Its optimal decision risk is
\begin{equation}
    \cD_{\cU}^*
    =
    \inf_{\vartheta\in\Theta_{\cU}}\cD(\vartheta).
\end{equation}
For any $\theta\in\Theta_{\cU}$, the total decision excess risk relative to the full parameter class admits decomposition
\begin{equation}
\underbrace{\cD(\theta)-\cD^*}_{\text{total decision excess risk}}
=
\underbrace{\cD(\theta)-\cD_{\cU}^*}_{\mathcal E_{\cD,\cU}(\theta)}
+
\underbrace{\cD_{\cU}^*-\cD^*}_{\mathcal A_{\cD}(\cU)}.
\label{eq:risk-decomposition}
\end{equation}
$\mathcal E_{\cD,\cU}(\theta)$ is restricted excess risk and $\mathcal A_{\cD}(\cU)$ is approximation error imposed by parameter budget.

\paragraph{Trajectory-induced losses and surrogate risk.}
Let $\tau\sim\nu$ be a trajectory-level training or adaptation example, such as a rollout or teacher-generated reasoning trace. It induces a differentiable supervised, preference-based, or reward-weighted loss $\ell_\tau(\theta)$, with population surrogate risk
\begin{equation}
    \cL(\theta)=\E_{\tau\sim\nu}[\ell_\tau(\theta)].
\end{equation}
Here ``surrogate'' distinguishes objective roles rather than supervision sources: $\ell_\tau$ supplies the update signal, whereas $\cD$ evaluates downstream decisions. The former may also use rewards, advantages, or preferences while differing from the latter through weighting, regularization, or sampling.

At an iterate $\theta\in\Theta_{\cU}$, feasible perturbations lie in space $\cU$. The decomposition
\[
\nabla\ell_\tau(\theta)
=
\Proj\nabla\ell_\tau(\theta)
+
(I-\Proj)\nabla\ell_\tau(\theta)
\]
separates the accessible component from one orthogonal to every feasible direction, where $I$ is the identity on $\R^p$ and $P_{\cU}$ is the orthogonal projector onto $\cU$. Therefore, define the accessible gradient for a single trajectory update as
\begin{equation}
    g_\tau=\Proj\nabla\ell_\tau(\theta).
\end{equation}
Since $g_\tau\in\cU$, a projected step of size $\eta>0$ (refers to learning rate) also remains feasible:
\begin{equation}
    \theta_\tau^+=\theta-\eta g_\tau\in\Theta_{\cU}.
    \label{eq:one-step-update}
\end{equation}

\paragraph{Learnability and decision utility.} Trajectory-induced reductions in population surrogate risk and population decision risk are called \emph{learnability} and \emph{decision utility}, respectively, defined as
\begin{align}
    \Lambda_\tau(\theta)
    =
    \cL(\theta)-\cL(\theta_\tau^+),
    \label{eq:learnability}
    \\
    U_\tau(\theta)
    =
    \cD(\theta)-\cD(\theta_\tau^+)
    =
    J(\theta_\tau^+)-J(\theta).
    \label{eq:decision-utility}
\end{align}
Each is positive when its objective improves. We ask when larger learnability implies larger decision utility under the restricted geometry $\cU$. Next, define the accessible population gradients
\begin{equation}
    g_{\cL}=\Proj\nabla\cL(\theta),
    \qquad
    g_{\cD}=\Proj\nabla\cD(\theta),
\end{equation}
and the local calibration gap
\begin{equation}
    \kappa_{\cU}(\theta)
    =
    \inf_{c\geq 0}
    \norm{g_{\cD}-c g_{\cL}}.
    \label{eq:calibration-gap}
\end{equation}
This gap measures how closely the surrogate and decision gradients align in the accessible update space, up to a nonnegative rescaling.

\section{Main Results}
\label{sec:main-results}

\subsection{One-step and cumulative surrogate-to-decision transfer}

\paragraph{First-order local expansions.}
For the common feasible perturbation
\begin{equation}
    \Delta_\tau
    =
    \theta_\tau^+-\theta
    =
    -\eta g_\tau,
    \label{eq:trajectory-perturbation}
\end{equation}
A first-order Taylor expansion under local smoothness gives
\[
\cL(\theta+\Delta_\tau)
=
\cL(\theta)
+
\ip{\nabla\cL(\theta)}{\Delta_\tau}
+
O\!\left(\norm{\Delta_\tau}^2\right).
\]
Substituting this expansion into learnability, and then
$\Delta_\tau=-\eta g_\tau$, yields
\begin{align}
\Lambda_\tau(\theta)
&=
\cL(\theta)-\cL(\theta+\Delta_\tau)
\nonumber\\
&=
-\ip{\nabla\cL(\theta)}{\Delta_\tau}
+O\!\left(\norm{\Delta_\tau}^2\right)
\nonumber\\
&=
\eta\ip{\nabla\cL(\theta)}{g_\tau}
+O\!\left(\eta^2\norm{g_\tau}^2\right)
\nonumber\\
&=
\eta\ip{g_{\cL}}{g_\tau}
+O\!\left(\eta^2\norm{g_\tau}^2\right).
\label{eq:learnability-first-order-expansion}
\end{align}
The last equality uses $g_\tau\in\cU$, so
$\ip{\nabla\cL(\theta)}{g_\tau}
=\ip{\Proj\nabla\cL(\theta)}{g_\tau}
=\ip{g_{\cL}}{g_\tau}$.
The same argument for $\cD$ gives
\begin{equation}
    U_\tau(\theta)
    =
    \eta\ip{g_{\cD}}{g_\tau}
    +O\!\left(\eta^2\norm{g_\tau}^2\right).
    \label{eq:utility-first-order-expansion}
\end{equation}
Therefore, for any fixed $c\geq0$,
\begin{equation}
    U_\tau(\theta)-c\Lambda_\tau(\theta)
    =
    \eta\ip{g_{\cD}-c g_{\cL}}{g_\tau}
    +
    O\!\left(\eta^2\norm{g_\tau}^2\right).
    \label{eq:first-order-transfer-expansion}
\end{equation}
The first-order discrepancy is governed by the interaction between $g_\tau$ and the calibrated residual $g_{\cD}-cg_{\cL}$. Theorem~\ref{thm:transfer} makes this comparison finite-step; Section~\ref{sec:decision-calibration} identifies when it preserves every accessible improving direction. Appendix~\ref{app:proof-transfer} gives the full remainder derivation.

\begin{theorem}[One-step transfer bound]
\label{thm:transfer}
Assume that $\cL$ and $\cD$ have $\beta_{\cL}$- and $\beta_{\cD}$-Lipschitz gradients, respectively, on a neighborhood containing every segment $[\theta,\theta_\tau^+]$. Then, for every trajectory $\tau$,
\begin{align}
    \left|
    \Lambda_\tau(\theta)
    -\eta\ip{g_{\cL}}{g_\tau}
    \right|
    &\leq
    \frac{\beta_{\cL}\eta^2}{2}\norm{g_\tau}^2,
    \label{eq:surrogate-taylor}\\
    \left|
    U_\tau(\theta)
    -\eta\ip{g_{\cD}}{g_\tau}
    \right|
    &\leq
    \frac{\beta_{\cD}\eta^2}{2}\norm{g_\tau}^2.
    \label{eq:decision-taylor}
\end{align}
Consequently, for every $c\geq0$,
\begin{equation}
\left|
U_\tau(\theta)-c\Lambda_\tau(\theta)
\right|
\leq
\eta\norm{g_{\cD}-c g_{\cL}}\norm{g_\tau}
+
\frac{\eta^2}{2}
\bigl(\beta_{\cD}+c\beta_{\cL}\bigr)
\norm{g_\tau}^2.
\label{eq:transfer-main}
\end{equation}
\end{theorem}

\emph{Proof.} See Appendix~\ref{app:proof-transfer}.

\paragraph{Optimal calibration scale.} Minimizing the first-order term in \eqref{eq:transfer-main} over $c\geq0$ gives
\begin{equation}
    c_{\cU}^*(\theta)
    \in
    \arg\min_{c\geq0}
    \norm{g_{\cD}-c g_{\cL}}.
    \label{eq:calibration-scale}
\end{equation}
For $g_{\cL}\neq0$, this is the projection coefficient of $g_{\cD}$ onto the nonnegative ray generated by $g_{\cL}$:
\begin{equation}
    c_{\cU}^*(\theta)
    =
    \frac{\bigl[\ip{g_{\cD}}{g_{\cL}}\bigr]_+}
    {\norm{g_{\cL}}^2},
    \qquad
    [a]_+=\max\{a,0\},
    \quad a\in\R.
    \label{eq:calibration-scale-closed-form}
\end{equation}
If $g_{\cL}=0$, all $c\geq0$ give the same mismatch, and we set $c_{\cU}^*(\theta)=0$. By definition,
\begin{equation}
    \norm{g_{\cD}-c_{\cU}^*(\theta)g_{\cL}}
    =
    \kappa_{\cU}(\theta),
    \label{eq:calibration-gap-at-optimum}
\end{equation}
Thus, \eqref{eq:transfer-main} separates the first-order calibration error $\eta\kappa_{\cU}(\theta)\norm{g_\tau}$ from finite-step curvature.

\begin{remark}[Sufficient and exact transfer]
\label{rem:positive-exact-transfer}
Theorem~\ref{thm:transfer} ensures $U_\tau(\theta)>0$ when calibrated learnability exceeds misalignment and curvature. If $\cD(\vartheta)=c\cL(\vartheta)+b$ along the segment for some $c>0$ and $b\in\R$, then $U_\tau(\theta)=c\Lambda_\tau(\theta)$ exactly. Appendix~\ref{app:positive-exact-transfer} formalizes both cases.
\end{remark}

\paragraph{Other optimizers.}
Conditional on a common optimizer state, the bound extends to any realized $\widetilde{\Delta}_\tau\in\cU$, including momentum \citep{polyak1964some}, Adam \citep{kingma2015adam}, AdamW \citep{loshchilov2019decoupled}, and Muon \citep{jordan2024muon,liu2025muon}. Replace $-\eta g_\tau$ by $\widetilde{\Delta}_\tau$; universal collinearity is unchanged (Appendix~\ref{app:proof-optimizer-update}).

\paragraph{Multi-step extension.}
Applying the one-step relation along a realized path compares cumulative surrogate risk and decision risk reductions; a common calibration scale makes both sums telescope.

\begin{corollary}[Cumulative transfer along an update path]
\label{cor:cumulative-transfer}
Starting from $\theta_0$, let
\begin{equation}
    g_t=\Proj\nabla\ell_{\tau_t}(\theta_t),
    \qquad
    \theta_{t+1}=\theta_t-\eta_t g_t,
    \qquad
    t=0,\ldots,T-1,
\end{equation}
where $\eta_t>0$ and the trajectories $\tau_t$ may be selected adaptively from the preceding iterates. Define
\begin{align}
    \Lambda_t
    &=
    \cL(\theta_t)-\cL(\theta_{t+1}),
    &
    U_t
    &=
    \cD(\theta_t)-\cD(\theta_{t+1}),\\
    g_{\cL,t}
    &=
    \Proj\nabla\cL(\theta_t),
    &
    g_{\cD,t}
    &=
    \Proj\nabla\cD(\theta_t).
\end{align}
Assume that $\cL$ and $\cD$ have $\beta_{\cL}$- and $\beta_{\cD}$-Lipschitz gradients, respectively, on a neighborhood containing every segment $[\theta_t,\theta_{t+1}]$. Then, for every fixed $c\geq0$,
\begin{align}
&
\left|
\bigl[\cD(\theta_0)-\cD(\theta_T)\bigr]
-
c\bigl[\cL(\theta_0)-\cL(\theta_T)\bigr]
\right|
\nonumber\\
&\quad\leq
\left|
\sum_{t=0}^{T-1}
\eta_t
\ip{g_{\cD,t}-c g_{\cL,t}}{g_t}
\right|
+
\frac{1}{2}
\sum_{t=0}^{T-1}
\eta_t^2
\bigl(\beta_{\cD}+c\beta_{\cL}\bigr)
\norm{g_t}^2
\nonumber\\
&\quad\leq
\sum_{t=0}^{T-1}
\eta_t
\norm{g_{\cD,t}-c g_{\cL,t}}
\norm{g_t}
+
\frac{1}{2}
\sum_{t=0}^{T-1}
\eta_t^2
\bigl(\beta_{\cD}+c\beta_{\cL}\bigr)
\norm{g_t}^2.
\label{eq:cumulative-transfer}
\end{align}
Let $B_T(c)$ denote the final right-hand side of \eqref{eq:cumulative-transfer}. Equivalently, the final restricted decision excess risk satisfies
\begin{equation}
\cD(\theta_T)-\cD_{\cU}^*
\leq
\cD(\theta_0)-\cD_{\cU}^*
-
c\bigl[\cL(\theta_0)-\cL(\theta_T)\bigr]
+
B_T(c).
\label{eq:cumulative-excess-risk}
\end{equation}
Thus, cumulative surrogate improvement guarantees cumulative decision improvement whenever
\begin{equation}
    c\bigl[\cL(\theta_0)-\cL(\theta_T)\bigr]>B_T(c).
\end{equation}
\end{corollary}

\emph{Proof.} See Appendix~\ref{app:proof-cumulative}.

The result also holds for realized optimizer displacements $\widetilde{\Delta}_t\in\cU$, conditional on their optimizer states, after replacing $\eta_tg_t$ by $-\widetilde{\Delta}_t$ and $\eta_t^2\norm{g_t}^2$ by $\norm{\widetilde{\Delta}_t}^2$.

\paragraph{Fixed versus local calibration scales.}
A common $c$ makes unweighted reductions telescope; iterate-specific $c_t^*$ instead yields a weighted bound. Both permit adaptive selection but compare one realized path; different terminal paths require stability assumptions (Appendix~\ref{app:local-calibration-scales}).

\subsection{When is learnability decision-calibrated?}\label{sec:decision-calibration}

The preceding bounds concern realized updates. Universal directional validity over the entire accessible space requires an exact condition.

\begin{theorem}[Necessary and sufficient condition for universal first-order calibration]
\label{thm:iff}
Assume $g_{\cL}\neq0$. The following statements are equivalent:
\begin{enumerate}[label=(\roman*),nosep,leftmargin=2em]
    \item $\ip{g_{\cL}}{v}>0$ implies $\ip{g_{\cD}}{v}>0$ for every $v\in\cU$.
    \item $g_{\cD}=c g_{\cL}$ for some scalar $c>0$.
\end{enumerate}
Thus, universal first-order transfer from learnability to decision improvement holds if and only if the two population gradients are positively collinear in the accessible update subspace.
\end{theorem}

\emph{Proof.} See Appendix~\ref{app:proof-iff}.

\begin{remark}[A miscalibration witness]
If Theorem~\ref{thm:iff} fails, some accessible direction improves the surrogate without improving decision risk to first order; unless $g_{\cD}=0$, it can strictly worsen decision risk. For a finite pool, positive collinearity on the span of its projected trajectory gradients is sufficient and, when transfer is required over that entire span, necessary. Transfer only on realized candidates requires the sign implication candidate by candidate.
\end{remark}

\textbf{\emph{Joint interpretation.}}
This gives a forward-looking training-time signal of deployment-time decision value with respect to the specified objective, characterized at two complementary levels:

(i) positive collinearity of the accessible surrogate and decision gradients is the exact condition for the signal to be universally valid to first order; and 

(ii) along the realized update path, cumulative surrogate improvement is guaranteed to transfer to decision utility when its calibrated gain outweighs the accumulated mismatch and curvature.

\subsection{Consequence for trajectory selection}

For a finite candidate pool, the same calibration gap bounds the decision utility lost by selecting trajectories according to learnability rather than a decision-aware oracle.

\begin{corollary}[Regret of learnability-based selection]
\label{cor:selection}
Under the assumptions of Theorem~\ref{thm:transfer}, let $\cT$ be a finite candidate pool satisfying $\norm{g_\tau}\leq G$ for every $\tau\in\cT$. Define
\begin{equation}
\begin{aligned}
    \tau_{\cL}\in\arg\max_{\tau\in\cT}\Lambda_\tau(\theta),
    \qquad
    \tau_{\cD}\in\arg\max_{\tau\in\cT}U_\tau(\theta),\\[-0.2em]
    c_{\cU}^*\in\arg\min_{c\geq0}\norm{g_{\cD}-cg_{\cL}}.
\end{aligned}
\end{equation}
Then
\begin{equation}
0
\leq
U_{\tau_{\cD}}(\theta)-U_{\tau_{\cL}}(\theta)
\leq
2\eta\kappa_{\cU}(\theta)G
+
\eta^2
\bigl(\beta_{\cD}+c_{\cU}^*\beta_{\cL}\bigr)G^2.
\label{eq:selection-bound}
\end{equation}
Hence, the bound guarantees near-optimal one-step decision improvement for learnability-based selection when the restricted calibration gap and the finite-step curvature term are both small.
\end{corollary}

\emph{Proof.} See Appendix~\ref{app:proof-selection}.

\begin{corollary}[Candidate-difference selection bound]
\label{cor:candidate-difference-selection}
Under the assumptions and notation of Corollary~\ref{cor:selection}, define
the candidate-difference space
\begin{equation}
    \mathcal V_{\cT}
    =
    \operatorname{span}
    \{g_\tau-g_\sigma:\tau,\sigma\in\cT\}
    \subseteq\cU
\end{equation}
and, for $c\geq0$, the worst pairwise first-order discrepancy
\begin{equation}
    \Gamma_{\cT}(c)
    =
    \max_{\tau,\sigma\in\cT}
    \left|
    \ip{
    P_{\mathcal V_{\cT}}(g_{\cD}-c g_{\cL})
    }{
    g_\tau-g_\sigma
    }
    \right|,
    \label{eq:candidate-difference-gap}
\end{equation}
where $P_{\mathcal V_{\cT}}$ is the orthogonal projector onto
$\mathcal V_{\cT}$. Then
\begin{align}
0
&\leq
U_{\tau_{\cD}}(\theta)-U_{\tau_{\cL}}(\theta)
\nonumber\\
&\leq
\inf_{c\geq0}
\left\{
\eta\Gamma_{\cT}(c)
+
\eta^2
\bigl(\beta_{\cD}+c\beta_{\cL}\bigr)G^2
\right\}
\nonumber\\
&\leq
2\eta\kappa_{\cU}(\theta)G
+
\eta^2
\bigl(\beta_{\cD}+c_{\cU}^*\beta_{\cL}\bigr)G^2.
\label{eq:candidate-difference-selection-bound}
\end{align}
Thus, full-space misalignment orthogonal to $\mathcal V_{\cT}$ does not
affect the first-order pairwise ranking discrepancy. Moreover, for any
$c>0$, $\tau_{\cL}$ is the unique decision-utility maximizer if, for every
$\sigma\neq\tau_{\cL}$,
\begin{align}
c\bigl(
\Lambda_{\tau_{\cL}}-\Lambda_\sigma
\bigr)
>
&\;
\eta
\left|
\ip{
P_{\mathcal V_{\cT}}(g_{\cD}-c g_{\cL})
}{
g_{\tau_{\cL}}-g_\sigma
}
\right|
\nonumber\\
&+
\frac{\eta^2}{2}
\bigl(\beta_{\cD}+c\beta_{\cL}\bigr)
\left(
\norm{g_{\tau_{\cL}}}^2+\norm{g_\sigma}^2
\right).
\label{eq:candidate-margin-certificate}
\end{align}
Hence, selection can remain stable despite a large full-space gap when
the residual is small on realized candidate differences or the learnability
margin is sufficiently large.
\end{corollary}

\emph{Proof.} See Appendix~\ref{app:proof-selection-difference}.

\textbf{\emph{Joint interpretation.}}
During model training, this gives an operational principle for learnability-based trajectory selection over a finite candidate pool under local smoothness and bounded candidate gradients:

(i) the decision-utility loss incurred by selecting the learnability maximizer is controlled by the full-space calibration gap and curvature; and

(ii) restricting attention to candidate-difference directions yields a tighter bound and a margin certificate for when the learnability winner is also the decision-utility winner.

\subsection{The role of the accessible update space}

Because the first-order calibration in Theorem~\ref{thm:transfer} depends on space $\cU$ (see Appendix~\ref{app:proof-subspace}), the trainable-parameter budget affects both capacity and alignment. Equivalently,
\begin{equation}
\kappa_{\cU}(\theta)
=
\inf_{c\geq0}
\norm{P_{\cU}\bigl(\nabla\cD(\theta)-c\nabla\cL(\theta)\bigr)}.
\end{equation}

\begin{proposition}[Approximation--calibration trade-off]
\label{prop:subspace}
Let $\cU_1\subseteq\cU_2$ be two update subspaces with the same base parameter $\theta_0$. Then we get
\begin{equation}
    \mathcal A_{\cD}(\cU_1)
    \geq
    \mathcal A_{\cD}(\cU_2),
    \qquad
    \kappa_{\cU_1}(\theta)
    \leq
    \kappa_{\cU_2}(\theta).
    \label{eq:subspace-tradeoff}
\end{equation}
Thus, a larger trainable space can lower the best attainable decision risk while exposing additional surrogate--decision mismatch.
\end{proposition}

\emph{Proof.} See Appendix~\ref{app:proof-subspace}.

\section{Experiments}
\label{sec:experiments}

We test the theory in a finite gridworld with exact returns and controlled alignment and at an LLM checkpoint. From a common checkpoint, candidates induce the projected-SGD update in \eqref{eq:one-step-update}, evaluated on fixed disjoint probes; Appendix~\ref{app:experimental-details} gives all relevant details. 

\subsection{Experimental Setup}
\label{sec:experimental-setup}

\paragraph{Synthetic RL.}
Using a PPO-initialized MLP in fully observed two-goal MiniGrid \citep{chevalier2023minigrid,schulman2017ppo}, we form behavioral-cloning updates from optimal, suboptimal, and wrong-goal demonstrations \citep{ross2011reduction}; their fraction $\rho$ controls misalignment. Decision risk is exact negative expected return, and nested coordinate spaces test update-space effects.

\paragraph{LLM post-training.} We analyze Qwen3.5-0.8B-Base \citep{qwen2026qwen35} on AQuA-RAT~\citep{ling2017program}, a five-choice mathematical-reasoning dataset with rationales. With rank-8 LoRA \citep{hu2022lora}, the surrogate is token-mean NLL on cleaned rationale and answer, while decision risk is negative correct-option probability from the five renormalized label logits conditioned on the gold rationale. Counterfactual corruption alters only surrogate-probe labels, while update-space experiments vary the final $k$ trainable blocks.

\subsection{One-Step and Cumulative Transfer}
\label{sec:empirical-transfer}

\begin{figure}[t]
\centering
\begin{minipage}[t]{0.31\linewidth}
\centering
\scriptsize\textbf{(a) Synthetic: one-step}\par\vspace{-0.5ex}
\includegraphics[width=\linewidth]{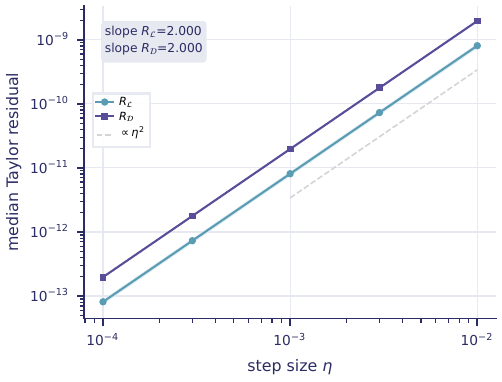}
\end{minipage}\hfill
\begin{minipage}[t]{0.31\linewidth}
\centering
\scriptsize\textbf{(b) Synthetic: cumulative}\par\vspace{-0.5ex}
\includegraphics[width=\linewidth]{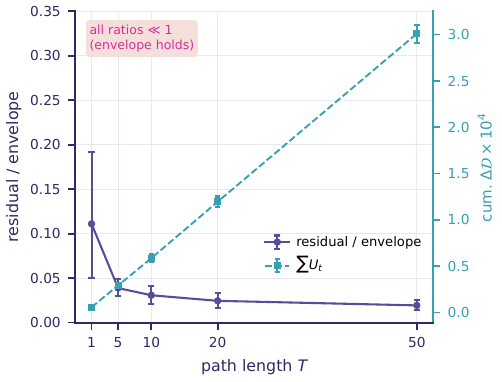}
\end{minipage}\hfill
\begin{minipage}[t]{0.31\linewidth}
\centering
\scriptsize\textbf{(c) LLM: one-step}\par\vspace{-0.5ex}
\includegraphics[width=\linewidth]{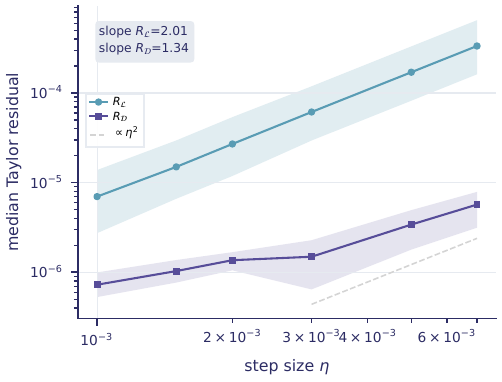}
\end{minipage}
\caption{\textbf{Local and pathwise transfer.}
One-step residuals in synthetic RL \textbf{(a)} and LLM post-training \textbf{(c)}, and synthetic pathwise discrepancies \textbf{(b)}. Synthetic \textbf{(b)} uses an empirical Hessian plug-in. Bands are seed-level 95\% bootstrap intervals.}
\label{fig:transfer}
\end{figure}

\paragraph{Synthetic RL.}
As shown in figure~\ref{fig:transfer}(a,b), both residuals have log--log slope $2.000$, matching Theorem~\ref{thm:transfer}; along repeated updates, the cumulative discrepancy remains below the plug-in envelope in every run and reaches a ratio of $0.019$ at $T=50$.

\begin{table}[t]
\centering
\caption{\textbf{LLM pathwise diagnostics.}
Five-seed medians with 95\% bootstrap intervals; quantities are defined in Appendix~\ref{app:cumulative-protocol}, and $Q_T$ and $N_T/A_T$ are scaled by $10^3$.}
\label{tab:llm-pathwise}
\scriptsize
\setlength{\tabcolsep}{5pt}
\begin{tabular}{@{}rccc@{}}
\toprule
$T$ & $Q_T(c)\;(10^{-3})$ & $Q_T(c)/|S_T(c)|$ & $N_T(c)/A_T(c)\;(10^{-3})$ \\
\midrule
$1$  & $0.0277\,[0.0147,\,0.191]$ & $0.269\,[0.068,\,0.538]$  & $14.48\,[9.32,\,24.87]$ \\
$5$  & $0.253\,[0.132,\,0.603]$  & $0.345\,[0.194,\,1.691]$  & $10.76\,[1.60,\,22.07]$ \\
$10$ & $0.695\,[0.148,\,6.093]$  & $1.170\,[0.298,\,33.105]$ & $7.21\,[4.19,\,20.16]$ \\
$20$ & $4.875\,[1.115,\,7.784]$  & $0.932\,[0.452,\,29.953]$ & $4.49\,[1.35,\,17.84]$ \\
$50$ & $9.150\,[3.378,\,21.820]$ & $1.191\,[0.929,\,1.993]$  & $4.42\,[1.64,\,16.62]$ \\
\bottomrule
\end{tabular}
\end{table}

\paragraph{LLM post-training.} See figure~\ref{fig:transfer}(c) and table~\ref{tab:llm-pathwise}. The surrogate residual is nearly quadratic (slope $2.01\,[1.90,\,2.15]$), whereas the decision estimate is noisier ($1.34\,[1.07,\,1.68]$ over four fitted seeds) near the numerical floor. First-order geometry is informative over short paths: $Q_T/|S_T|$ is $0.27$ at $T=1$ and $0.35$ at $T=5$, but reaches $1.19\,[0.93,\,1.99]$ at $T=50$ as higher-order effects become comparable to the signed term; wide intermediate intervals reflect cancellation. Meanwhile $N_T/A_T$ falls from $1.45\times10^{-2}$ to $4.4\,[1.6,\,16.6]\times10^{-3}$. Thus the short-horizon account is informative, but without a curvature envelope the long-path result remains descriptive rather than a certificate for \eqref{eq:cumulative-transfer}.

\subsection{Calibration Geometry and Misalignment Witnesses}
\label{sec:empirical-calibration}

\begin{figure}[t]
\centering
\begin{minipage}[t]{0.275\linewidth}
\centering
\scriptsize\textbf{(a) Calibration geometry}\par\vspace{-0.5ex}
\includegraphics[width=\linewidth]{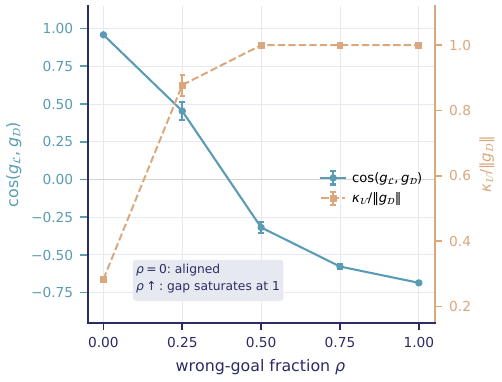}
\end{minipage}\hfill
\begin{minipage}[t]{0.275\linewidth}
\centering
\scriptsize\textbf{(b) Misalignment witness}\par\vspace{-0.5ex}
\includegraphics[width=\linewidth]{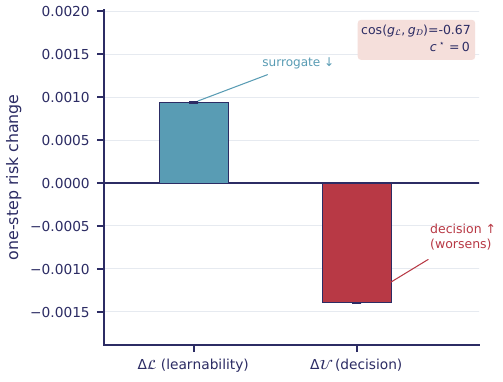}
\end{minipage}\hfill
\begin{minipage}[t]{0.275\linewidth}
\centering
\scriptsize\textbf{(c) Trajectory selection}\par\vspace{-0.5ex}
\includegraphics[width=\linewidth]{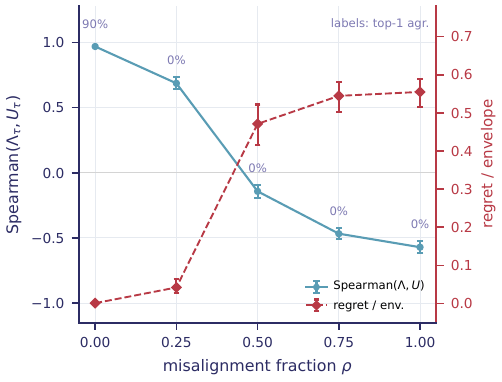}
\end{minipage}
\caption{\textbf{Calibration failure and selection in synthetic RL.}
Increasing $\rho$ rotates the gradients, produces a decision-harmful witness, and degrades learnability-based ranking.}
\label{fig:synthetic-misalignment}
\end{figure}

\begin{figure}[t]
\centering
\begin{minipage}[t]{0.25\linewidth}
\centering
\scriptsize\textbf{(a) Synthetic RL}\par\vspace{-0.5ex}
\includegraphics[width=\linewidth]{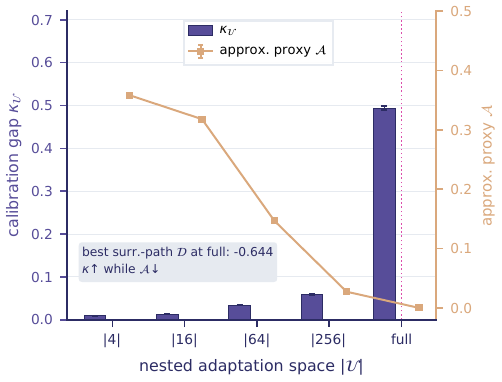}
\end{minipage}\hfill
\begin{minipage}[t]{0.25\linewidth}
\centering
\scriptsize\textbf{(b) LLM calibration}\par\vspace{-0.5ex}
\includegraphics[width=\linewidth]{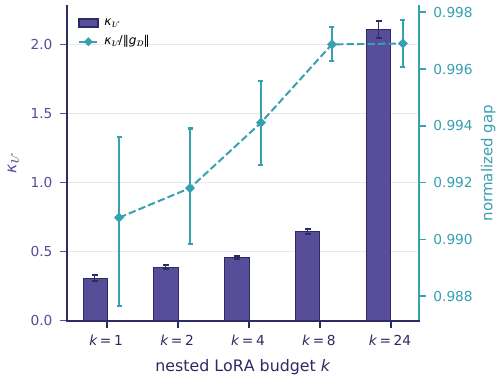}
\end{minipage}\hfill
\begin{minipage}[t]{0.25\linewidth}
\centering
\scriptsize\textbf{(c) LLM approximation}\par\vspace{-0.5ex}
\includegraphics[width=\linewidth]{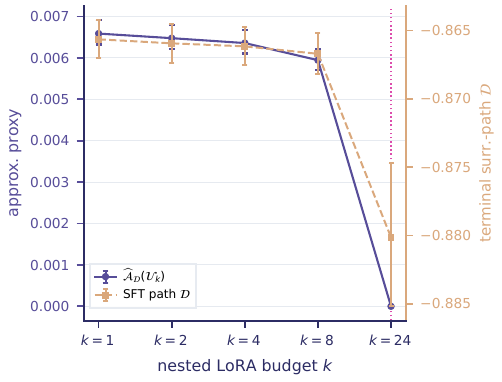}
\end{minipage}
\caption{\textbf{Update-space trade-off.}
Synthetic nested policy spaces \textbf{(a)} and nested LLM LoRA masks \textbf{(b,c)}. Comparisons share a checkpoint; largest-space proxies are zero by definition.}
\label{fig:update-space}
\end{figure}

\paragraph{Synthetic RL.}
As $\rho$ grows, cosine alignment falls from $0.96$ to $-0.69$ and the normalized gap rises from $0.28$ to $1.00$; the resulting witness satisfies $\Lambda(v)=9.4\times10^{-4}>0>U(v)=-1.4\times10^{-3}$ in all 20 seeds, as predicted when Theorem~\ref{thm:iff} fails (Figure~\ref{fig:synthetic-misalignment}(a,b)).

\begin{table}[t]
\caption{\textbf{LLM calibration and cross-fitted witnesses.}
Five-seed medians with 95\% bootstrap intervals (Appendix~\ref{app:llm-crossfit}).}
\label{tab:llm-calibration}
\centering
\scriptsize
\renewcommand{\arraystretch}{1.00}
\begin{minipage}[t]{0.5\linewidth}
\centering
\textbf{(a) Calibration geometry}\par\vspace{0.5ex}
\setlength{\tabcolsep}{5pt}
\resizebox{\linewidth}{!}{%
\begin{tabular}{@{}lcc@{}}
\toprule
Statistic
& Natural ($\rho=0$)
& Counterfactual ($\rho=\rho_{\max}$) \\
\midrule
$\cos(g_{\cL},g_{\cD})$
& $0.074\,[0.069,\,0.087]$
& $-0.605\,[-0.721,\,-0.577]$ \\
$\kappa_{\cU}/\norm{g_{\cD}}$
& $0.997\,[0.996,\,0.998]$
& $1.000\,[1.000,\,1.000]$ \\
$c_{\cU}^*$
& $0.249\,[0.210,\,0.285]$
& $0.000\,[0.000,\,0.000]$ \\
\bottomrule
\end{tabular}}
\end{minipage}\hfill
\begin{minipage}[t]{0.48\linewidth}
\centering
\textbf{(b) Held-out witness}\par\vspace{0.5ex}
\setlength{\tabcolsep}{5pt}
\resizebox{\linewidth}{!}{%
\begin{tabular}{@{}lc@{}}
\toprule
Statistic & Counterfactual ($\rho=\rho_{\max}$) \\
\midrule
Held-out surrogate reduction $\Lambda(v)$
& $1.34\,[1.24,\,1.48]\times10^{-3}$ \\
Held-out decision utility $U(v)$
& $-4.65\,[-5.88,\,-4.54]\times10^{-3}$ \\
Seeds with $\Lambda(v)>0>U(v)$
& $5/5$ \\
\bottomrule
\end{tabular}}
\end{minipage}
\end{table}

\paragraph{LLM post-training.} As shown in Table~\ref{tab:llm-calibration}(a), natural supervision is nearly orthogonal to the decision gradient (cosine $0.074$; normalized gap $0.997$), while corruption makes the gradients oppose (cosine $-0.605$) and sets $c_{\cU}^*=0$. In all five seeds, the cross-fitted witness remains harmful out of fold after a finite update: $\Lambda(v)=1.34\times10^{-3}>0>U(v)=-4.65\times10^{-3}$ (Table~\ref{tab:llm-calibration}(b)).

\subsection{Trajectory Selection}
\label{sec:empirical-selection}

\paragraph{Synthetic RL.}
As $\rho$ grows, Spearman correlation falls from $0.97$ to $-0.57$, regret rises from $5.4\times10^{-9}$ to $9.4\times10^{-6}$, and top-one agreement falls from $0.90$ to zero, showing that controlled misalignment reverses ranking reliability (Figure~\ref{fig:synthetic-misalignment}(c)).

\paragraph{LLM post-training.}
Corruption lowers Spearman$(\Lambda,U)$ from $0.15$ to $-0.17$, whereas the held-out oracle gap remains a nonmonotone $8.72$--$9.16\times10^{-4}$. Only at $\rho=0.75$ does a rank crossing lower utility; no seed matches the oracle, and references stay fixed because their candidate bank is unchanged (Table~\ref{tab:llm-selection}).

\begin{table}[t]
\centering
\caption{\textbf{Cross-fitted LLM trajectory selection.}
Five-seed medians with 95\% bootstrap intervals; panel~(b) gives $\rho$-invariant references.}
\label{tab:llm-selection}
\scriptsize
\begin{minipage}[t]{0.63\linewidth}
\centering
\textbf{(a) Corruption-dependent selection}\par\vspace{0.5ex}
\resizebox{\linewidth}{!}{%
\begin{tabular}{@{}crrrrrr@{}}
\toprule
& \multicolumn{2}{c}{Spearman$(\Lambda,U)$}
& \multicolumn{2}{c}{Learnability-selected $U$ ($10^{-4}$)}
& \multicolumn{2}{c}{Held-out oracle gap ($10^{-4}$)} \\
\cmidrule(lr){2-3}\cmidrule(lr){4-5}\cmidrule(l){6-7}
$\rho$ & Median & Bootstrap CI
& Median & Bootstrap CI
& Median & Bootstrap CI \\
\midrule
$0$
& $0.149$ & $[0.007,\,0.511]$
& $1.36$ & $[0.38,\,3.21]$
& $8.72$ & $[5.56,\,9.16]$ \\
$0.25$
& $0.070$ & $[-0.122,\,0.349]$
& $1.36$ & $[0.38,\,3.21]$
& $8.72$ & $[5.56,\,9.16]$ \\
$0.50$
& $-0.068$ & $[-0.265,\,0.215]$
& $1.36$ & $[0.38,\,3.21]$
& $8.72$ & $[5.56,\,9.16]$ \\
$0.75$
& $-0.170$ & $[-0.357,\,0.110]$
& $0.40$ & $[-4.68,\,1.36]$
& $9.16$ & $[5.56,\,13.82]$ \\
\bottomrule
\end{tabular}}
\end{minipage}\hfill
\begin{minipage}[t]{0.35\linewidth}
\centering
\textbf{(b) $\rho$-invariant references}\par\vspace{0.5ex}
\resizebox{\linewidth}{!}{%
\begin{tabular}{@{}lrr@{}}
\toprule
Selector & Median $U$ & Bootstrap CI \\
& \multicolumn{2}{c}{($10^{-4}$)} \\
\midrule
Random & $-0.867$ & $[-1.45,\,2.37]$ \\
Gradient-cosine & $0.765$ & $[-0.790,\,11.61]$ \\
Decision oracle & $9.40$ & $[5.96,\,11.27]$ \\
\bottomrule
\end{tabular}}
\end{minipage}
\end{table}

\paragraph{Candidate-difference refinement.}
The natural LLM full-space gap is $0.997$, yet regret is nonmonotone (Table~\ref{tab:llm-calibration}(a) and Table~\ref{tab:llm-selection}(a)). Using the candidate-difference quantities in Corollary~\ref{cor:candidate-difference-selection}, Table~\ref{tab:candidate-difference-selection} gives refined/full ratios of $0.31$--$0.60$ (synthetic) and $0.027$--$0.031$ (LLM FO). Under the tie-aware plateau-margin check, the synthetic diagnostic passes in $19/20$ seeds at $\rho=0$ but in $0/20$ once $\rho\geq0.25$, while the LLM FO diagnostic passes in $0/5$ throughout. Thus, candidate differences remove substantial ranking-invisible residuals, but passing disappears once ranking mismatch remains.

\begin{table}[t]
\centering
\caption{\textbf{Candidate-difference selection diagnostics.}
Each $\rho$ group reports the seed median, 95\% bootstrap interval,
and tie-aware plateau-margin-check pass count. Synthetic RL uses plug-in envelopes including curvature,
whereas LLM uses first-order terms.}
\label{tab:candidate-difference-selection}
\scriptsize
\setlength{\tabcolsep}{2.5pt}
\renewcommand{\arraystretch}{1.15}
\resizebox{\linewidth}{!}{%
\begin{tabular}{@{}l*{5}{ccc}@{}}
\toprule
&
\multicolumn{3}{c}{$\rho=0$}
&
\multicolumn{3}{c}{$\rho=0.25$}
&
\multicolumn{3}{c}{$\rho=0.50$}
&
\multicolumn{3}{c}{$\rho=0.75$}
&
\multicolumn{3}{c}{$\rho=1.00$}
\\
\cmidrule(lr){2-4}
\cmidrule(lr){5-7}
\cmidrule(lr){8-10}
\cmidrule(lr){11-13}
\cmidrule(l){14-16}
Setting and ratio
& Median & Interval & Pass
& Median & Interval & Pass
& Median & Interval & Pass
& Median & Interval & Pass
& Median & Interval & Pass
\\
\midrule
\shortstack[l]{Synthetic RL\\
$B_{\mathrm{diff}}/B_{\mathrm{full}}$}
& $0.311$ & $[0.292,\,0.357]$ & $19/20$
& $0.470$ & $[0.441,\,0.544]$ & $0/20$
& $0.603$ & $[0.540,\,0.616]$ & $0/20$
& $0.603$ & $[0.542,\,0.616]$ & $0/20$
& $0.603$ & $[0.542,\,0.616]$ & $0/20$
\\
\shortstack[l]{LLM\\
$B_{\mathrm{diff}}^{\mathrm{FO}}/
 B_{\mathrm{full}}^{\mathrm{FO}}$}
& $0.027$ & $[0.008,\,0.033]$ & $0/5$
& $0.028$ & $[0.009,\,0.033]$ & $0/5$
& $0.030$ & $[0.009,\,0.035]$ & $0/5$
& $0.031$ & $[0.010,\,0.036]$ & $0/5$
& --- & --- & ---
\\
\bottomrule
\end{tabular}%
}
\end{table}

\subsection{Update-Space Trade-off}
\label{sec:empirical-adaptation}

\paragraph{Synthetic RL.}
Expanding the space lowers the approximation proxy from $0.36$ to zero and raises $\kappa_{\cU}$ from $0.009$ to $0.49$; the full space has the lowest terminal risk in all 20 seeds (Figure~\ref{fig:update-space}(a)).

\paragraph{LLM post-training.}
Across LoRA budgets, raw $\kappa_{\cU}$ rises from $0.30$ to $2.12$ while the approximation proxy falls from $6.47\times10^{-3}$ to zero; $k=24$ has the lowest terminal risk in all five seeds. This qualitatively supports Proposition~\ref{prop:subspace}, although the proxy is not $\mathcal A_{\cD}$ (Figure~\ref{fig:update-space}(b,c)).

\section{Limitations}
\label{sec:limitations}

The guarantees are local to fixed objectives, an iterate, restricted update space, and realized path. They neither establish finite-sample generalization from probe-estimated calibration to future deployment distributions nor cover approximate, noisy, or randomized top-$k$ selection, selector-induced path comparisons, or changing on-policy distributions. At LLM scale, the required smoothness constants are not certified, so the Taylor and pathwise results remain diagnostics. Finally, the selection bounds are sufficient and potentially conservative rather than converses, even after candidate-difference refinement.

\section{Conclusion}
\label{sec:conclusion}

We developed a trajectory-wise local theory of when surrogate-driven updates improve downstream decisions. The one-step bound separates accessible-gradient misalignment from curvature, its cumulative extension tracks these errors along a realized path, and, when the accessible surrogate gradient is nonzero, positive collinearity exactly characterizes universal first-order transfer. The same geometry bounds the decision regret of learnability-based selection, admits a sharper candidate-difference refinement, and reveals an approximation--calibration trade-off across trainable subspaces. Controlled gridworld and LLM experiments are consistent with the predicted Taylor scaling, failures under misalignment, degradation of candidate rankings, and update-space tension. Extending these local diagnostics into distributionally valid training-time signals under finite probes, adaptive selection, and on-policy shift is the central next step. More broadly, whether surrogate progress becomes decision progress is determined not by the surrogate alone, but by its alignment with the decision objective in the directions the model can actually move.

{\small
\bibliographystyle{plainnat}
\bibliography{reference}
}

\appendix
\normalsize
\section{Proofs}

\subsection{Proof of Theorem~\ref{thm:transfer}}
\label{app:proof-transfer}

\begin{proof}
Let $F:\R^p\to\R$ have a $\beta$-Lipschitz gradient on the relevant update segment, and set
\[
    \delta_\tau
    =
    \theta_\tau^+-\theta
    =
    -\eta g_\tau.
\]
To derive the first-order remainder bound, define the scalar function
\[
    \phi(t)=F(\theta+t\delta_\tau),
    \qquad t\in[0,1].
\]
By the chain rule,
\[
    \phi'(t)
    =
    \ip{\nabla F(\theta+t\delta_\tau)}{\delta_\tau}.
\]
The fundamental theorem of calculus therefore gives
\[
\begin{aligned}
    F(\theta+\delta_\tau)-F(\theta)
    &=
    \phi(1)-\phi(0)\\
    &=
    \int_0^1
    \ip{\nabla F(\theta+t\delta_\tau)}{\delta_\tau}
    \,dt.
\end{aligned}
\]
Subtracting the first-order term
$\ip{\nabla F(\theta)}{\delta_\tau}$ gives
\begin{align}
&
F(\theta+\delta_\tau)-F(\theta)
-\ip{\nabla F(\theta)}{\delta_\tau}
\nonumber\\
&\quad=
\int_0^1
\ip{\nabla F(\theta+t\delta_\tau)}{\delta_\tau}
\,dt
-
\ip{\nabla F(\theta)}{\delta_\tau}
\nonumber\\
&\quad=
\int_0^1
\ip{\nabla F(\theta+t\delta_\tau)}{\delta_\tau}
\,dt
-
\int_0^1
\ip{\nabla F(\theta)}{\delta_\tau}
\,dt
\nonumber\\
&\quad=
\int_0^1
\left[
\ip{\nabla F(\theta+t\delta_\tau)}{\delta_\tau}
-
\ip{\nabla F(\theta)}{\delta_\tau}
\right]
\,dt
\nonumber\\
&\quad=
\int_0^1
\ip{
\nabla F(\theta+t\delta_\tau)-\nabla F(\theta)
}{
\delta_\tau
}
\,dt.
\label{eq:integral-taylor-remainder}
\end{align}

Taking absolute values, applying the triangle inequality for
integrals, and then Cauchy--Schwarz gives
\[
\begin{aligned}
&
\left|
F(\theta+\delta_\tau)-F(\theta)
-\ip{\nabla F(\theta)}{\delta_\tau}
\right|
\\
&\quad=
\left|
\int_0^1
\ip{
\nabla F(\theta+t\delta_\tau)-\nabla F(\theta)
}{
\delta_\tau
}
\,dt
\right|
\\
&\quad\leq
\int_0^1
\left|
\ip{
\nabla F(\theta+t\delta_\tau)-\nabla F(\theta)
}{
\delta_\tau
}
\right|
\,dt
\\
&\quad\leq
\int_0^1
\norm{
\nabla F(\theta+t\delta_\tau)-\nabla F(\theta)
}
\norm{\delta_\tau}
\,dt.
\end{aligned}
\]
Because $\nabla F$ is $\beta$-Lipschitz,
\[
\norm{
\nabla F(\theta+t\delta_\tau)-\nabla F(\theta)
}
\leq
\beta t\norm{\delta_\tau}.
\]
Substituting this inequality into the preceding display gives
\[
\begin{aligned}
&
\left|
F(\theta+\delta_\tau)-F(\theta)
-\ip{\nabla F(\theta)}{\delta_\tau}
\right|\\
&\qquad\leq
\int_0^1
\beta t\norm{\delta_\tau}^2\,dt
=
\frac{\beta}{2}\norm{\delta_\tau}^2.
\end{aligned}
\]
Hence,
\begin{equation}
\left|
F(\theta+\delta_\tau)-F(\theta)
-\ip{\nabla F(\theta)}{\delta_\tau}
\right|
\leq
\frac{\beta}{2}\norm{\delta_\tau}^2.
\label{eq:generic-taylor}
\end{equation}

Since $\theta_\tau^+=\theta+\delta_\tau$, multiplying the expression inside the absolute value in
\eqref{eq:generic-taylor} by $-1$ gives
\[
\left|
\bigl[F(\theta)-F(\theta_\tau^+)\bigr]
+
\ip{\nabla F(\theta)}{\delta_\tau}
\right|
\leq
\frac{\beta}{2}\norm{\delta_\tau}^2.
\]
Now substitute $\delta_\tau=-\eta g_\tau$. Because
\[
\ip{\nabla F(\theta)}{\delta_\tau}
=
-\eta\ip{\nabla F(\theta)}{g_\tau}
\]
and
\[
\norm{\delta_\tau}^2
=
\eta^2\norm{g_\tau}^2,
\]
we obtain
\begin{equation}
\left|
\bigl[F(\theta)-F(\theta_\tau^+)\bigr]
-
\eta\ip{\nabla F(\theta)}{g_\tau}
\right|
\leq
\frac{\beta\eta^2}{2}\norm{g_\tau}^2.
\label{eq:generic-improvement}
\end{equation}

Because $g_\tau\in\cU$ and $\Proj$ is the orthogonal projector onto $\cU$, the component of $\nabla F(\theta)$ orthogonal to $\cU$ has zero inner product with $g_\tau$. Therefore,
\[
\begin{aligned}
\ip{\nabla F(\theta)}{g_\tau}
&=
\ip{
\Proj\nabla F(\theta)
+
(I-\Proj)\nabla F(\theta)
}{
g_\tau
}\\
&=
\ip{\Proj\nabla F(\theta)}{g_\tau}.
\end{aligned}
\]
Applying \eqref{eq:generic-improvement} first to
$F=\cL$ with $\beta=\beta_{\cL}$, and then to
$F=\cD$ with $\beta=\beta_{\cD}$, gives
\eqref{eq:surrogate-taylor} and
\eqref{eq:decision-taylor}.

Define the corresponding approximation errors by
\[
    e_{\cL}
    =
    \Lambda_\tau
    -
    \eta\ip{g_{\cL}}{g_\tau},
    \qquad
    e_{\cD}
    =
    U_\tau
    -
    \eta\ip{g_{\cD}}{g_\tau}.
\]
Equations~\eqref{eq:surrogate-taylor} and
\eqref{eq:decision-taylor} imply
\[
    |e_{\cL}|
    \leq
    \frac{\beta_{\cL}\eta^2}{2}\norm{g_\tau}^2,
    \qquad
    |e_{\cD}|
    \leq
    \frac{\beta_{\cD}\eta^2}{2}\norm{g_\tau}^2.
\]
For any $c\geq0$,
\[
\begin{aligned}
U_\tau-c\Lambda_\tau
&=
\eta\ip{g_{\cD}}{g_\tau}
+e_{\cD}
-c\left(
\eta\ip{g_{\cL}}{g_\tau}
+e_{\cL}
\right)\\
&=
\eta\ip{g_{\cD}-c g_{\cL}}{g_\tau}
+e_{\cD}-c e_{\cL}.
\end{aligned}
\]
Taking absolute values and applying the triangle inequality and
Cauchy--Schwarz yields
\[
\begin{aligned}
|U_\tau-c\Lambda_\tau|
&\leq
\eta
\norm{g_{\cD}-c g_{\cL}}
\norm{g_\tau}
+
|e_{\cD}|+c|e_{\cL}|\\
&\leq
\eta
\norm{g_{\cD}-c g_{\cL}}
\norm{g_\tau}\\
&\qquad+
\frac{\eta^2}{2}
\bigl(
\beta_{\cD}+c\beta_{\cL}
\bigr)
\norm{g_\tau}^2,
\end{aligned}
\]
which is exactly \eqref{eq:transfer-main}.
\end{proof}

\subsection{Positive and exact transfer consequences}
\label{app:positive-exact-transfer}

Substituting $c_{\cU}^*(\theta)$ into \eqref{eq:transfer-main} and using
\eqref{eq:calibration-gap-at-optimum} gives the explicit lower bound
\begin{equation}
\begin{aligned}
    U_\tau(\theta)
    \geq\;&
    c_{\cU}^*(\theta)\Lambda_\tau(\theta)
    -
    \eta\kappa_{\cU}(\theta)\norm{g_\tau}\\
    &-
    \frac{\eta^2}{2}
    \bigl(
    \beta_{\cD}
    +
    c_{\cU}^*(\theta)\beta_{\cL}
    \bigr)
    \norm{g_\tau}^2.
\end{aligned}
\label{eq:one-step-positive-transfer}
\end{equation}
Therefore, $U_\tau(\theta)>0$ whenever the right-hand side of
\eqref{eq:one-step-positive-transfer} is positive.

For the exact case, suppose that for some $c>0$ and $b\in\R$,
\begin{equation}
    \cD(\vartheta)
    =
    c\cL(\vartheta)+b
    \qquad
    \text{for every }\vartheta\in[\theta,\theta_\tau^+].
    \label{eq:exact-objective-alignment}
\end{equation}
Equivalently, $J(\vartheta)=-c\cL(\vartheta)-b$ on the update segment. Direct substitution gives
\begin{equation}
\begin{aligned}
    U_\tau(\theta)
    &=
    \cD(\theta)-\cD(\theta_\tau^+)\\
    &=
    c\bigl[\cL(\theta)-\cL(\theta_\tau^+)\bigr]\\
    &=
    c\Lambda_\tau(\theta).
\end{aligned}
\label{eq:exact-transfer}
\end{equation}
Thus the actual finite-step transfer discrepancy vanishes in this special case, even though the generic curvature bound need not be tight. This is the exact-alignment case summarized in Remark~\ref{rem:positive-exact-transfer}.

\subsection{Extension to optimizer-generated updates}
\label{app:proof-optimizer-update}

Fix the current optimizer state, and let an arbitrary update rule applied to trajectory $\tau$ produce
\[
    \widetilde{\theta}_\tau^+
    =
    \theta+\widetilde{\Delta}_\tau,
    \qquad
    \widetilde{\Delta}_\tau\in\cU.
\]
Define
\[
    \widetilde{\Lambda}_\tau
    =
    \cL(\theta)-\cL(\widetilde{\theta}_\tau^+),
    \qquad
    \widetilde{U}_\tau
    =
    \cD(\theta)-\cD(\widetilde{\theta}_\tau^+).
\]
If $\cL$ and $\cD$ have $\beta_{\cL}$- and $\beta_{\cD}$-Lipschitz gradients on the segment $[\theta,\widetilde{\theta}_\tau^+]$, then, for every $c\geq0$,
\begin{equation}
\left|
\widetilde{U}_\tau-c\widetilde{\Lambda}_\tau
\right|
\leq
\norm{g_{\cD}-c g_{\cL}}
\norm{\widetilde{\Delta}_\tau}
+
\frac{1}{2}
\bigl(\beta_{\cD}+c\beta_{\cL}\bigr)
\norm{\widetilde{\Delta}_\tau}^2.
\label{eq:optimizer-transfer}
\end{equation}
Thus, momentum \citep{polyak1964some}, Adam \citep{kingma2015adam}, AdamW \citep{loshchilov2019decoupled}, and matrix-orthogonalized Muon \citep{jordan2024muon,liu2025muon} enter the same local comparison through their realized displacement, provided that candidate updates are evaluated from the same optimizer state and remain in $\cU$. Because Theorem~\ref{thm:iff} quantifies all directions in $\cU$, its universal positive-collinearity characterization is unchanged.

\begin{proof}
For $F\in\{\cL,\cD\}$, apply \eqref{eq:generic-taylor} with $\delta_\tau=\widetilde{\Delta}_\tau$. Since $\widetilde{\Delta}_\tau\in\cU$,
\[
\ip{\nabla F(\theta)}{\widetilde{\Delta}_\tau}
=
\ip{\Proj\nabla F(\theta)}{\widetilde{\Delta}_\tau}.
\]
Consequently,
\[
\left|
\widetilde{\Lambda}_\tau
+
\ip{g_{\cL}}{\widetilde{\Delta}_\tau}
\right|
\leq
\frac{\beta_{\cL}}{2}\norm{\widetilde{\Delta}_\tau}^2,
\qquad
\left|
\widetilde{U}_\tau
+
\ip{g_{\cD}}{\widetilde{\Delta}_\tau}
\right|
\leq
\frac{\beta_{\cD}}{2}\norm{\widetilde{\Delta}_\tau}^2.
\]
Subtracting $c$ times the first relation from the second, followed by the triangle inequality and Cauchy--Schwarz, proves \eqref{eq:optimizer-transfer}. The argument uses only the realized displacement and therefore does not require differentiability of the update rule itself.
\end{proof}

\subsection{Proof of Corollary~\ref{cor:cumulative-transfer}}
\label{app:proof-cumulative}

\begin{proof}
Apply Theorem~\ref{thm:transfer} at each iterate $\theta_t$ with
trajectory $\tau_t$ and step size $\eta_t$. Define the corresponding
one-step Taylor errors by
\[
    e_{\cL,t}
    =
    \Lambda_t
    -
    \eta_t\ip{g_{\cL,t}}{g_t},
    \qquad
    e_{\cD,t}
    =
    U_t
    -
    \eta_t\ip{g_{\cD,t}}{g_t}.
\]
Theorem~\ref{thm:transfer} gives
\[
    |e_{\cL,t}|
    \leq
    \frac{\beta_{\cL}\eta_t^2}{2}\norm{g_t}^2,
    \qquad
    |e_{\cD,t}|
    \leq
    \frac{\beta_{\cD}\eta_t^2}{2}\norm{g_t}^2.
\]
Equivalently,
\[
    \Lambda_t
    =
    \eta_t\ip{g_{\cL,t}}{g_t}
    +
    e_{\cL,t},
    \qquad
    U_t
    =
    \eta_t\ip{g_{\cD,t}}{g_t}
    +
    e_{\cD,t}.
\]
Therefore, for any fixed $c\geq0$,
\begin{align}
    U_t-c\Lambda_t
    &=
    \eta_t\ip{g_{\cD,t}}{g_t}
    +e_{\cD,t}
    -
    c\left(
    \eta_t\ip{g_{\cL,t}}{g_t}
    +e_{\cL,t}
    \right)
    \nonumber\\
    &=
    \eta_t
    \ip{g_{\cD,t}-c g_{\cL,t}}{g_t}
    +
    e_{\cD,t}-c e_{\cL,t}.
    \label{eq:pathwise-one-step-decomposition}
\end{align}

Summing \eqref{eq:pathwise-one-step-decomposition} over
$t=0,\ldots,T-1$ gives
\begin{align}
\sum_{t=0}^{T-1}
\bigl(U_t-c\Lambda_t\bigr)
&=
\sum_{t=0}^{T-1}
\eta_t
\ip{g_{\cD,t}-c g_{\cL,t}}{g_t}
\nonumber\\
&\quad+
\sum_{t=0}^{T-1}
\bigl(e_{\cD,t}-c e_{\cL,t}\bigr).
\label{eq:pathwise-summed-decomposition}
\end{align}
Because the same population objectives $\cL$ and $\cD$ are evaluated
along one realized path, both sequences telescope:
\[
\begin{aligned}
\sum_{t=0}^{T-1}U_t
&=
\sum_{t=0}^{T-1}
\bigl[
\cD(\theta_t)-\cD(\theta_{t+1})
\bigr]\\
&=
\cD(\theta_0)-\cD(\theta_T),
\end{aligned}
\]
and similarly,
\[
\sum_{t=0}^{T-1}\Lambda_t
=
\cL(\theta_0)-\cL(\theta_T).
\]
Substituting these identities into
\eqref{eq:pathwise-summed-decomposition} yields
\begin{align}
&
\bigl[\cD(\theta_0)-\cD(\theta_T)\bigr]
-
c\bigl[\cL(\theta_0)-\cL(\theta_T)\bigr]
\nonumber\\
&\quad=
\sum_{t=0}^{T-1}
\eta_t
\ip{g_{\cD,t}-c g_{\cL,t}}{g_t}
+
\sum_{t=0}^{T-1}
\bigl(e_{\cD,t}-c e_{\cL,t}\bigr).
\label{eq:pathwise-telescoped-decomposition}
\end{align}

Taking absolute values and applying the triangle inequality gives
\begin{align}
&
\left|
\bigl[\cD(\theta_0)-\cD(\theta_T)\bigr]
-
c\bigl[\cL(\theta_0)-\cL(\theta_T)\bigr]
\right|
\nonumber\\
&\quad\leq
\left|
\sum_{t=0}^{T-1}
\eta_t
\ip{g_{\cD,t}-c g_{\cL,t}}{g_t}
\right|
+
\sum_{t=0}^{T-1}
\left|
e_{\cD,t}-c e_{\cL,t}
\right|.
\label{eq:pathwise-triangle}
\end{align}
Since $c\geq0$,
\[
\begin{aligned}
\left|
e_{\cD,t}-c e_{\cL,t}
\right|
&\leq
|e_{\cD,t}|+c|e_{\cL,t}|\\
&\leq
\frac{\eta_t^2}{2}
\bigl(
\beta_{\cD}+c\beta_{\cL}
\bigr)
\norm{g_t}^2.
\end{aligned}
\]
Substituting this bound into \eqref{eq:pathwise-triangle} gives the
first inequality in \eqref{eq:cumulative-transfer}.

For the remaining first-order term, the triangle inequality for sums
and Cauchy--Schwarz give
\begin{align}
\left|
\sum_{t=0}^{T-1}
\eta_t
\ip{g_{\cD,t}-c g_{\cL,t}}{g_t}
\right|
&\leq
\sum_{t=0}^{T-1}
\eta_t
\left|
\ip{g_{\cD,t}-c g_{\cL,t}}{g_t}
\right|
\nonumber\\
&\leq
\sum_{t=0}^{T-1}
\eta_t
\norm{g_{\cD,t}-c g_{\cL,t}}
\norm{g_t}.
\end{align}
Combining the preceding inequalities proves both inequalities in
\eqref{eq:cumulative-transfer}.

Let $B_T(c)$ denote the final right-hand side of
\eqref{eq:cumulative-transfer}. The resulting absolute-value bound is
equivalent to
\[
-B_T(c)
\leq
\bigl[\cD(\theta_0)-\cD(\theta_T)\bigr]
-
c\bigl[\cL(\theta_0)-\cL(\theta_T)\bigr]
\leq
B_T(c).
\]
Using its lower side gives
\[
\cD(\theta_0)-\cD(\theta_T)
\geq
c\bigl[\cL(\theta_0)-\cL(\theta_T)\bigr]
-
B_T(c).
\]
Rearranging for the terminal decision risk yields
\[
\cD(\theta_T)
\leq
\cD(\theta_0)
-
c\bigl[\cL(\theta_0)-\cL(\theta_T)\bigr]
+
B_T(c).
\]
Subtracting the constant $\cD_{\cU}^*$ from both sides gives
\[
\cD(\theta_T)-\cD_{\cU}^*
\leq
\cD(\theta_0)-\cD_{\cU}^*
-
c\bigl[\cL(\theta_0)-\cL(\theta_T)\bigr]
+
B_T(c),
\]
which is \eqref{eq:cumulative-excess-risk}.

Finally, if
\[
c\bigl[\cL(\theta_0)-\cL(\theta_T)\bigr]
>
B_T(c),
\]
then
\[
\cD(\theta_T)
<
\cD(\theta_0).
\]
Hence the cumulative surrogate improvement guarantees a strict
cumulative improvement in decision risk under the stated condition.
\end{proof}

\subsection{Local calibration scales along a path}
\label{app:local-calibration-scales}

The common $c$ in Corollary~\ref{cor:cumulative-transfer} is needed for the unweighted surrogate reductions to telescope. If instead
\[
c_t^*
\in
\arg\min_{c\geq0}
\norm{g_{\cD,t}-c g_{\cL,t}}
\]
is chosen separately at each iterate, applying Theorem~\ref{thm:transfer} stepwise gives
\begin{align}
&
\left|
\cD(\theta_0)-\cD(\theta_T)
-
\sum_{t=0}^{T-1}c_t^*\Lambda_t
\right|
\nonumber\\
&\quad\leq
\sum_{t=0}^{T-1}
\eta_t\kappa_{\cU}(\theta_t)\norm{g_t}
+
\frac{1}{2}
\sum_{t=0}^{T-1}
\eta_t^2
\bigl(\beta_{\cD}+c_t^*\beta_{\cL}\bigr)
\norm{g_t}^2.
\label{eq:cumulative-local-scale}
\end{align}
This bounds a weighted cumulative surrogate reduction rather than a fixed multiple of $\cL(\theta_0)-\cL(\theta_T)$. Both forms permit adaptive selection but compare objectives along the same realized path; comparing terminal points reached by different selectors requires additional stability assumptions.

\subsection{Proof of Theorem~\ref{thm:iff}}
\label{app:proof-iff}

\begin{proof}
Statement (ii) immediately implies (i). Conversely, decompose $g_{\cD}=c g_{\cL}+b$ with $b\perp g_{\cL}$. If $b\neq0$, then for $v=g_{\cL}-\lambda b$ and sufficiently large $\lambda$, one has $\ip{g_{\cL}}{v}>0$ but $\ip{g_{\cD}}{v}<0$, contradicting (i). Hence $b=0$. Taking $v=g_{\cL}$ then forces $c>0$.
\end{proof}

\subsection{Proof of Corollary~\ref{cor:selection}}
\label{app:proof-selection}

\begin{proof}
Set $c^*=c_{\cU}^*$. By Theorem~\ref{thm:transfer}, every
$\tau\in\cT$ satisfies
\[
|U_\tau-c^*\Lambda_\tau|
\leq
\eta\kappa_{\cU}(\theta)\norm{g_\tau}
+
\frac{\eta^2}{2}
\bigl(
\beta_{\cD}+c^*\beta_{\cL}
\bigr)
\norm{g_\tau}^2.
\]
Since $\norm{g_\tau}\leq G$ uniformly over $\cT$, define
\[
\delta
=
\eta\kappa_{\cU}(\theta)G
+
\frac{\eta^2}{2}
\bigl(
\beta_{\cD}+c^*\beta_{\cL}
\bigr)G^2.
\]
Then, for every candidate $\tau\in\cT$,
\[
c^*\Lambda_\tau-\delta
\leq
U_\tau
\leq
c^*\Lambda_\tau+\delta.
\]

Applying the upper inequality to the decision-utility maximizer
$\tau_{\cD}$ gives
\[
U_{\tau_{\cD}}
\leq
c^*\Lambda_{\tau_{\cD}}+\delta.
\]
By definition, $\tau_{\cL}$ maximizes learnability over $\cT$, so
\[
\Lambda_{\tau_{\cD}}
\leq
\Lambda_{\tau_{\cL}}.
\]
Because $c^*\geq0$, multiplying this inequality by $c^*$ preserves
its direction:
\[
c^*\Lambda_{\tau_{\cD}}
\leq
c^*\Lambda_{\tau_{\cL}}.
\]
Hence,
\[
U_{\tau_{\cD}}
\leq
c^*\Lambda_{\tau_{\cL}}+\delta.
\]

Applying the lower uniform inequality to $\tau_{\cL}$ gives
\[
U_{\tau_{\cL}}
\geq
c^*\Lambda_{\tau_{\cL}}-\delta,
\]
or equivalently,
\[
c^*\Lambda_{\tau_{\cL}}
\leq
U_{\tau_{\cL}}+\delta.
\]
Combining the preceding inequalities yields
\[
U_{\tau_{\cD}}
\leq
U_{\tau_{\cL}}+2\delta.
\]
Finally, because $\tau_{\cD}$ maximizes decision utility,
\[
U_{\tau_{\cD}}-U_{\tau_{\cL}}\geq0.
\]
Therefore,
\[
\begin{aligned}
0
&\leq
U_{\tau_{\cD}}-U_{\tau_{\cL}}\\
&\leq
2\delta\\
&=
2\eta\kappa_{\cU}(\theta)G
+
\eta^2
\bigl(
\beta_{\cD}+c_{\cU}^*\beta_{\cL}
\bigr)G^2,
\end{aligned}
\]
which proves \eqref{eq:selection-bound}.
\end{proof}

\subsection{Proof of Corollary~\ref{cor:candidate-difference-selection}}
\label{app:proof-selection-difference}

\begin{proof}
For $c\geq0$ and each candidate $\tau\in\cT$, write
\[
    r_c=g_{\cD}-c g_{\cL},
    \qquad
    U_\tau-c\Lambda_\tau
    =
    \eta\ip{r_c}{g_\tau}+e_\tau(c).
\]
The two Taylor bounds in Theorem~\ref{thm:transfer} imply
\begin{equation}
    |e_\tau(c)|
    \leq
    \frac{\eta^2}{2}
    \bigl(\beta_{\cD}+c\beta_{\cL}\bigr)
    \norm{g_\tau}^2.
    \label{eq:candidate-error-bound}
\end{equation}
Subtracting the representations for any $\tau,\sigma\in\cT$ gives
\begin{align}
&
\bigl(U_\tau-U_\sigma\bigr)
-
c\bigl(\Lambda_\tau-\Lambda_\sigma\bigr)
\nonumber\\
&\qquad
=
\eta\ip{r_c}{g_\tau-g_\sigma}
+
e_\tau(c)-e_\sigma(c).
\label{eq:candidate-pairwise-decomposition}
\end{align}
Because $g_\tau-g_\sigma\in\mathcal V_{\cT}$, its inner product with
the component of $r_c$ orthogonal to $\mathcal V_{\cT}$ is zero. Hence,
\[
    \ip{r_c}{g_\tau-g_\sigma}
    =
    \ip{
    P_{\mathcal V_{\cT}}r_c
    }{
    g_\tau-g_\sigma
    }.
\]

Set $\tau=\tau_{\cD}$ and $\sigma=\tau_{\cL}$ in
\eqref{eq:candidate-pairwise-decomposition}. Since $\tau_{\cL}$
maximizes learnability and $c\geq0$,
\[
    c\bigl(
    \Lambda_{\tau_{\cD}}-\Lambda_{\tau_{\cL}}
    \bigr)
    \leq0.
\]
Using \eqref{eq:candidate-error-bound}, the definition of
$\Gamma_{\cT}(c)$, and $\norm{g_\tau}\leq G$ therefore gives
\[
\begin{aligned}
U_{\tau_{\cD}}-U_{\tau_{\cL}}
&\leq
\eta\Gamma_{\cT}(c)
+
|e_{\tau_{\cD}}(c)|
+
|e_{\tau_{\cL}}(c)|\\
&\leq
\eta\Gamma_{\cT}(c)
+
\eta^2
\bigl(\beta_{\cD}+c\beta_{\cL}\bigr)G^2.
\end{aligned}
\]
This holds for every $c\geq0$, so taking the infimum proves the first
upper bound in \eqref{eq:candidate-difference-selection-bound}.

To compare it with the full-space bound, evaluate the refined expression
at $c=c_{\cU}^*$. By Cauchy--Schwarz, contraction of orthogonal projection,
and $\norm{g_\tau-g_\sigma}\leq2G$,
\[
\begin{aligned}
\Gamma_{\cT}(c_{\cU}^*)
&\leq
\norm{
P_{\mathcal V_{\cT}}
\bigl(g_{\cD}-c_{\cU}^*g_{\cL}\bigr)
}
\max_{\tau,\sigma\in\cT}
\norm{g_\tau-g_\sigma}\\
&\leq
2G
\norm{g_{\cD}-c_{\cU}^*g_{\cL}}
=
2G\kappa_{\cU}(\theta).
\end{aligned}
\]
Substitution gives the final inequality in
\eqref{eq:candidate-difference-selection-bound}; nonnegativity follows
from the definition of $\tau_{\cD}$.

Finally, fix $\sigma\neq\tau_{\cL}$ and rearrange
\eqref{eq:candidate-pairwise-decomposition} with
$\tau=\tau_{\cL}$. Equations
\eqref{eq:candidate-error-bound} and
\eqref{eq:candidate-margin-certificate} imply
\[
    U_{\tau_{\cL}}-U_\sigma>0.
\]
If the condition holds for every such $\sigma$, then
$\tau_{\cL}$ uniquely maximizes decision utility over $\cT$.
\end{proof}

\subsection{Proof of Proposition~\ref{prop:subspace}}
\label{app:proof-subspace}

\begin{proof}
We prove the two inequalities separately.

First, because $\cU_1\subseteq\cU_2$ and the two accessible
parameter classes have the same base point $\theta_0$,
\[
\Theta_{\cU_1}
=
\theta_0+\cU_1
\subseteq
\theta_0+\cU_2
=
\Theta_{\cU_2}.
\]
Minimizing the same decision risk over the larger feasible set cannot
produce a larger optimal value. Therefore,
\[
\cD_{\cU_1}^*
=
\inf_{\vartheta\in\Theta_{\cU_1}}\cD(\vartheta)
\geq
\inf_{\vartheta\in\Theta_{\cU_2}}\cD(\vartheta)
=
\cD_{\cU_2}^*.
\]
Subtracting the unrestricted optimum $\cD^*$ from both sides gives
\[
\begin{aligned}
\mathcal A_{\cD}(\cU_1)
&=
\cD_{\cU_1}^*-\cD^*\\
&\geq
\cD_{\cU_2}^*-\cD^*\\
&=
\mathcal A_{\cD}(\cU_2).
\end{aligned}
\]
This proves the approximation-error inequality.

We next compare the calibration gaps. For any $z\in\R^p$, the
nesting $\cU_1\subseteq\cU_2$ gives the orthogonal decomposition
\[
P_{\cU_2}z
=
P_{\cU_1}z
+
\bigl(P_{\cU_2}-P_{\cU_1}\bigr)z.
\]
Here,
\[
P_{\cU_1}z\in\cU_1
\]
by the definition of orthogonal projection. Moreover,
$P_{\cU_2}z\in\cU_2$ and
$P_{\cU_1}z\in\cU_1\subseteq\cU_2$, so
\[
\bigl(P_{\cU_2}-P_{\cU_1}\bigr)z
\in\cU_2.
\]
To show that this difference is also orthogonal to $\cU_1$, take any
$u\in\cU_1$. Since $\cU_1\subseteq\cU_2$, the defining property of
orthogonal projections gives
\[
\ip{P_{\cU_2}z}{u}
=
\ip{z}{u},
\qquad
\ip{P_{\cU_1}z}{u}
=
\ip{z}{u}.
\]
Therefore,
\[
\begin{aligned}
\ip{
\bigl(P_{\cU_2}-P_{\cU_1}\bigr)z
}{u}
&=
\ip{P_{\cU_2}z}{u}
-
\ip{P_{\cU_1}z}{u}\\
&=0.
\end{aligned}
\]
Because this holds for every $u\in\cU_1$,
\[
\bigl(P_{\cU_2}-P_{\cU_1}\bigr)z
\in
\cU_1^\perp.
\]
Combining the two inclusions yields
\[
\bigl(P_{\cU_2}-P_{\cU_1}\bigr)z
\in
\cU_2\cap\cU_1^\perp.
\]
Thus, the two terms in
\[
P_{\cU_2}z
=
P_{\cU_1}z
+
\bigl(P_{\cU_2}-P_{\cU_1}\bigr)z
\]
are orthogonal. By the Pythagorean identity,
\[
\norm{P_{\cU_2}z}^2
=
\norm{P_{\cU_1}z}^2
+
\norm{
\bigl(P_{\cU_2}-P_{\cU_1}\bigr)z
}^2.
\]
Consequently,
\[
\norm{P_{\cU_1}z}
\leq
\norm{P_{\cU_2}z}.
\]

For each $c\geq0$, set
\[
z(c)
=
\nabla\cD(\theta)-c\nabla\cL(\theta).
\]
Applying the preceding projection inequality to $z(c)$ gives,
pointwise for every $c\geq0$,
\[
\norm{
P_{\cU_1}
\bigl(
\nabla\cD(\theta)-c\nabla\cL(\theta)
\bigr)
}
\leq
\norm{
P_{\cU_2}
\bigl(
\nabla\cD(\theta)-c\nabla\cL(\theta)
\bigr)
}.
\]
Taking the infimum over the same set of scales $c\geq0$ on both
sides preserves the inequality:
\[
\begin{aligned}
\kappa_{\cU_1}(\theta)
&=
\inf_{c\geq0}
\norm{
P_{\cU_1}
\bigl(
\nabla\cD(\theta)-c\nabla\cL(\theta)
\bigr)
}\\
&\leq
\inf_{c\geq0}
\norm{
P_{\cU_2}
\bigl(
\nabla\cD(\theta)-c\nabla\cL(\theta)
\bigr)
}\\
&=
\kappa_{\cU_2}(\theta).
\end{aligned}
\]
This proves the calibration-gap inequality and completes the proof.
\end{proof}

\section{Experimental Details}
\label{app:experimental-details}

\paragraph{Basic measurement protocol.}
Candidate, surrogate-probe, and decision-probe sets are disjoint, and the fixed probe averages satisfy
\[
    g_F=\Proj\nabla F(\theta),
    \qquad
    F\in\{\cL,\cD\}.
\]
Analyzed updates use projected SGD without momentum, weight decay, clipping, or optimizer state. Dropout is disabled and gradient quantities use FP32. LLM runs additionally disable TF32, hold the LoRA mask fixed within each comparison, and form $\Lambda_\tau$ and $U_\tau$ from paired per-example FP64 risk differences. Candidate pools contain 64 synthetic or 32 LLM trajectories.

All experiments were conducted on a single NVIDIA RTX A6000 GPU with 48\,GB of memory.

\subsection{Evaluation Protocols}

Both settings evaluate one-step transfer, pathwise accumulation, a miscalibration witness, trajectory selection, and nested update spaces.

\paragraph{One-step transfer.}
At a common checkpoint, each candidate induces one projected update for each step size in the sweep. We measure the Taylor residual
\[
R_F(\eta,\tau)
=
\left|
F(\theta)-F(\theta-\eta g_\tau)
-\eta\ip{g_F}{g_\tau}
\right|,
\qquad F\in\{\cL,\cD\}.
\]
The log--log slope of median $R_F$ against $\eta$ tests second-order scaling.

\paragraph{Cumulative transfer.}
\label{app:cumulative-protocol}
Starting from the same checkpoint, we apply a fixed candidate sequence for up to 50 updates and recompute both gradients at every iterate. For prefixes $T\in\{1,5,10,20,50\}$ and fixed $c=c_{\cU}^*(\theta_0)$, define the signed discrepancy
\begin{equation}
    E_T(c)
    =
    \bigl[\cD(\theta_0)-\cD(\theta_T)\bigr]
    -
    c\bigl[\cL(\theta_0)-\cL(\theta_T)\bigr]
    ,
    \qquad
    N_T(c)=|E_T(c)|,
    \label{eq:empirical-cumulative-residual}
\end{equation}
and the non-negative first-order scale
\begin{equation}
    A_T(c)
    =
    \sum_{t=0}^{T-1}
    \eta_t
    \norm{g_{\cD,t}-c g_{\cL,t}}
    \norm{g_t}.
    \label{eq:empirical-first-order-scale}
\end{equation}
For the LLM, we also compute the signed first-order sum
\begin{equation}
    S_T(c)
    =
    \sum_{t=0}^{T-1}
    \eta_t
    \ip{g_{\cD,t}-c g_{\cL,t}}{g_t},
    \label{eq:signed-first-order-sum}
\end{equation}
and the accumulated remainder $Q_T(c)=|E_T(c)-S_T(c)|$. Unlike $S_T$, $A_T$ does not permit cancellation. Synthetic RL estimates a plug-in curvature envelope; the LLM quantities $N_T/A_T$, $Q_T$, and $Q_T/|S_T|$ are descriptive.

\paragraph{Miscalibration witness.}
We construct an accessible direction with positive first-order surrogate improvement and negative first-order decision utility, then test whether the sign pattern persists after a finite update. Let
\[
b
=
g_{\cD}
-
\frac{\ip{g_{\cD}}{g_{\cL}}}{\norm{g_{\cL}}^2}g_{\cL}.
\]
When $b\neq0$, we use
\[
v_\lambda=g_{\cL}-\lambda b,
\qquad
\lambda>\frac{\ip{g_{\cD}}{g_{\cL}}}{\norm{b}^2}.
\]
In both settings, we set
$\lambda=\ip{g_{\cD}}{g_{\cL}}/\norm{b}^{2}
+\norm{g_{\cL}}/\norm{b}$ and rescale $v_\lambda$ to the median
candidate-gradient norm before applying the finite update.

\paragraph{Trajectory selection.}
We vary the fraction of misaligned supervision $\rho$ while holding the checkpoint, candidate updates, and decision probe fixed. For a selector $s$ returning $\tau_s\in\cT$, define
\[
    R_{\mathrm{sel}}(s)
    =
    U_{\tau_{\cD}}-U_{\tau_s},
    \qquad
    G=\max_{\tau\in\cT}\norm{g_\tau}.
\]
We report selection regret, selected utility, top-one agreement with the decision-utility oracle, and Spearman$(\Lambda_\tau,U_\tau)$. The oracle is used only for evaluation.

For Table~\ref{tab:candidate-difference-selection}, define
\[
\begin{aligned}
B_{\mathrm{full}}
&=
2\eta\kappa_{\cU}G
+
\eta^2
\bigl(\beta_{\cD}+c_{\cU}^*\beta_{\cL}\bigr)G^2,\\
B_{\mathrm{diff}}
&=
\inf_{c\geq0}
\left\{
\eta\Gamma_{\cT}(c)
+
\eta^2
\bigl(\beta_{\cD}+c\beta_{\cL}\bigr)G^2
\right\},\\
B_{\mathrm{full}}^{\mathrm{FO}}
&=
2\eta\kappa_{\cU}G,
\qquad
B_{\mathrm{diff}}^{\mathrm{FO}}
=
\eta\inf_{c\geq0}\Gamma_{\cT}(c).
\end{aligned}
\]
Here, $B_{\mathrm{full}}$ is the right-hand side of
\eqref{eq:selection-bound}, equivalently the final right-hand side of
\eqref{eq:candidate-difference-selection-bound}, whereas
$B_{\mathrm{diff}}$ is the middle right-hand side of
\eqref{eq:candidate-difference-selection-bound}. The full-space quantity
uses $c=c_{\cU}^*$, while the candidate-difference quantity chooses
$c\geq0$ to minimize its displayed objective. For synthetic RL,
$\beta_{\cL}$ and $\beta_{\cD}$ are replaced by the corresponding plug-in
curvature estimates $\widehat\beta_{\cL}$ and $\widehat\beta_{\cD}$; the
LLM reports the FO variants, obtained by dropping the curvature terms.

For the candidate-difference diagnostic, candidates within $\epsilon_\Lambda=\max\{10^{-7},10^{-4}|\max_\sigma\Lambda_\sigma|\}$ of the maximum learnability form a numerical plateau. For this retrospective check, we use its decision-best member and report a tie-aware pass only if the decision oracle lies in the plateau and every comparison with a candidate outside it satisfies the relevant margin inequality up to slack tolerance $\epsilon_s=\max\{10^{-12},10^{-2}\eta^2\}$; the LLM check retains only first-order terms.

\paragraph{Update-space trade-off.}
From the same checkpoint, we compare nested accessible spaces. For each space, we measure $\kappa_{\cU}$, a finite-budget proxy for restricted attainable decision risk, and the terminal decision risk reached by surrogate training.

\subsection{Synthetic RL: Implementation and Measurements}

\paragraph{Implementation.}
The synthetic task is a fully observed $5\times5$, horizon-25 two-goal gridworld. A two-layer, 64-unit $\tanh$ policy with a four-action softmax head is pretrained by PPO for $8000$ environment steps. Each candidate is an episode from an optimal, suboptimal, or wrong-goal teacher, with behavioral-cloning loss averaged over its state--action pairs. Decision risk is the negative expected episodic return, differentiated exactly through finite-state occupancy recursion under the known dynamics. We use
\[
\eta\in\{10^{-4},3\times10^{-4},10^{-3},3\times10^{-3},10^{-2}\},
\qquad
\rho\in\{0,.25,.5,.75,1\},
\]
where $\rho$ is the fraction of wrong-goal supervision.
Unless stated otherwise, cumulative, witness, and selection analyses use $\eta=10^{-3}$ and $|\cU|=64$.

\paragraph{One-step and cumulative measurements.}
The one-step evaluation uses $R_F$. For cumulative transfer, projected Hessian--vector products and power iteration estimate directional smoothness at 101 equally spaced points per update segment, giving
\begin{equation}
    \widehat B_T(c)
    =
    A_T(c)
    +
    \frac{1}{2}
    \sum_{t=0}^{T-1}
    \eta_t^2
    \bigl(\widehat\beta_{\cD,t}
    +c\widehat\beta_{\cL,t}\bigr)
    \norm{g_t}^2,
    \label{eq:empirical-plugin-rhs}
\end{equation}
where $\widehat\beta_{F,t}$ is the largest projected spectral-norm estimate for objective $F$ over the sampled points of $[\theta_t,\theta_{t+1}]$. With $\epsilon=10^{-12}$, we report
\[
    \mathrm{ratio}^{\mathrm{syn}}_T
    =
    \frac{N_T(c)}{\widehat B_T(c)+\epsilon}.
\]
Because the grid is finite and the eigensolver approximate, $\widehat B_T(c)$ is diagnostic rather than certified.

\paragraph{Witness and selection measurements.}
When $b\neq0$, $v_\lambda$ improves the surrogate and worsens decision risk to first order; if $b=0$ and $g_{\cD}$ is a nonpositive multiple of $g_{\cL}$, we use $v=g_{\cL}$. Selection regret is normalized by
\[
    \widehat S_{\mathrm{sel}}
    =
    2\eta\kappa_{\cU}G
    +
    \eta^2
    \bigl(
    \widehat\beta_{\cD}
    +c_{\cU}^*\widehat\beta_{\cL}
    \bigr)G^2,
\]
the plug-in counterpart of \eqref{eq:selection-bound}.

\paragraph{Update-space measurements.}
At the common checkpoint $\theta_0$, we compute $\kappa_{\cU_k}$ over
\[
|\cU|\in\{4,16,64,256,\mathrm{full}\}.
\]
We estimate restricted attainable risk by direct decision-risk optimization over five restarts and record the final projected-gradient norm as a stationarity diagnostic.

\subsection{LLM Post-Training: Implementation and Measurements}

\paragraph{Data and objectives.}
Disjoint AQuA-RAT training subsets provide 10K warm-start examples, a 2K candidate reservoir, and separate 512-example surrogate and decision probes. For the update-space proxy, the decision probe is split into 256 optimization and 256 held-out examples. For question and choices $x$, gold rationale $r$, and correct option $y$, let $\bar r$ remove any trailing answer reveal from $r$, and let $\mathcal C(x,\bar r,y)$ index completion tokens in $\bar r$ followed by $y$. Masking the input $x$, the surrogate loss is
\begin{equation}
    \ell_{(x,\bar r,y)}(\theta)
    =
    -\frac{1}{|\mathcal C(x,\bar r,y)|}
    \sum_{j\in\mathcal C(x,\bar r,y)}
    \log p_\theta(w_j\mid x,w_{<j}).
    \label{eq:llm-surrogate-loss}
\end{equation}
This is mean completion NLL without rationale--answer reweighting.

For $a\in\{\mathrm A,\ldots,\mathrm E\}$, let $s_\theta(a\mid x,\bar r)$ be the option-label log-likelihood under the fixed question--choice--rationale context, and define
\begin{equation}
    q_\theta(a\mid x,\bar r)
    =
    \frac{\exp s_\theta(a\mid x,\bar r)}
    {\sum_{a'\in\{\mathrm A,\ldots,\mathrm E\}}
    \exp s_\theta(a'\mid x,\bar r)}.
    \label{eq:llm-label-probability}
\end{equation}
The empirical decision risk on probe $S_{\cD}$ is
\begin{equation}
    \widehat{\cD}_{S_{\cD}}(\theta)
    =
    -\frac{1}{|S_{\cD}|}
    \sum_{(x,\bar r,y)\in S_{\cD}}
    q_\theta(y\mid x,\bar r).
    \label{eq:llm-decision-risk}
\end{equation}
Five-way $\arg\max_a q_\theta(a\mid x,\bar r)$ accuracy is secondary; because it conditions on the gold rationale, it is not an end-to-end generation result.

\paragraph{Model and optimization.}
Each seed warm-starts for 25 AdamW steps (about 400 examples) at learning rate $2\times10^{-5}$, effective batch size 16 (microbatch 4 with four accumulation steps), zero weight decay, and sequence length 512, without decision-based stopping or selection. Rank-8 LoRA uses $\alpha=16$, zero dropout, and no bias on attention or feed-forward projections. Within a seed, all comparisons share $\theta_0$ and probes; seeds vary warm-start order and candidate/counterfactual sampling.

\paragraph{One-step measurements.}
The step-size sweep is
\[
\eta\in\{10^{-3},1.5\times10^{-3},2\times10^{-3},3\times10^{-3},5\times10^{-3},7\times10^{-3}\},
\]
a fixed grid chosen to clear the FP32 residual floor while remaining below the higher-order regime near $\eta\approx10^{-2}$. The development split selects the largest step satisfying
\[
\frac{
|\Lambda_{\mathrm{dev}}-\eta\,\lVert g_{\mathrm{dev}}\rVert^{2}|
}{
|\Lambda_{\mathrm{dev}}|
}
\leq0.1.
\]
where $g_{\mathrm{dev}}=\Proj\nabla\cL_{\mathrm{dev}}(\theta_0)$ and
$\Lambda_{\mathrm{dev}}
=\cL_{\mathrm{dev}}(\theta_0)
-\cL_{\mathrm{dev}}(\theta_0-\eta g_{\mathrm{dev}})$. This rule selects $\eta=7\times10^{-3}$ in all five seeds and fixes it thereafter. Decision probes and test accuracy are not used for tuning; the test split serves only the secondary accuracy report. Eight fixed candidates per seed fit the residual slope, while all 32 evaluate ranking at the selected step. A point enters the fit only when $R_F$ exceeds the repeated-evaluation numerical floor.

\paragraph{Cumulative measurements.}
Each seed follows one fixed 50-candidate path and reports its prefixes, recomputing both gradients at every iterate with $c=c_{\cU}^*(\theta_0)$ fixed. Because no LLM curvature envelope is estimated, we report
\begin{equation}
    \mathrm{ratio}^{\mathrm{llm}}_T
    =
    \frac{N_T(c)}{A_T(c)+\epsilon}
    \label{eq:llm-cumulative-ratio}
\end{equation}
using the non-negative scale $A_T(c)$ in \eqref{eq:empirical-first-order-scale}; Table~\ref{tab:llm-pathwise} also reports $Q_T(c)$ and $Q_T/|S_T(c)|$ from \eqref{eq:signed-first-order-sum}. All three are descriptive diagnostics.

\paragraph{Misalignment and selection measurements.}
For the controlled-misalignment, witness, and selection analyses,
\[
\rho\in\{0,0.25,0.5,0.75\}.
\]
Nested subsets of the same surrogate probe replace the gold terminal option with a uniformly sampled wrong option; questions, rationales, the candidate reservoir, and the gold decision probe remain fixed. Thus, $\rho$ changes only $\cL_\rho$, $g_{\cL_\rho}$, and candidate learnability.

Let
\[
    g_{\mathrm{ref}}
    =
    P_{\cU}\nabla\cL_0(\theta_0).
\]
The fixed-reference gradient-cosine baseline scores a candidate by
\begin{equation}
    s_{\mathrm{grad}}(\tau)
    =
    \frac{\ip{g_\tau}{g_{\mathrm{ref}}}}
    {\norm{g_\tau}\norm{g_{\mathrm{ref}}}+\epsilon}.
    \label{eq:llm-gradient-cosine}
\end{equation}
This is a one-step analogue of LESS, not its full pipeline, and its reference is fixed across $\rho$. Random selection uses the same candidates and budget; the oracle $\tau_{\cD}$ is used only to evaluate regret. LLM regret is normalized by $2\eta\kappa_{\cU}G+\epsilon$.

\paragraph{Witness measurements.}
At $\rho=0.75$ and $k=8$, where $b\neq0$, we choose
\[
\lambda>\frac{\ip{g_{\cD}}{g_{\cL}}}{\norm{b}^2},
\]
rescale $v_\lambda$ to the median candidate-gradient norm, and apply $-\eta v_\lambda$. We report its first-order signs and finite-step changes. The constructed direction illustrates Theorem~\ref{thm:iff} but need not coincide with an observed candidate update.

\paragraph{Update-space measurements.}
Masks expose the final $k\in\{1,2,4,8,24\}$ blocks of one adapter initialization. Each $k$ restarts from $\theta_0$ with the same trajectories, step size, update count, and budget. For this proxy, write $S_{\cD}=S_{\cD}^{\mathrm{opt}}\sqcup S_{\cD}^{\mathrm{ho}}$ for the fixed 256-example optimization and held-out halves. The surrogate path applies 20 projected-SGD updates; $\kappa_{\cU_k}(\theta_0)$ uses the fixed surrogate probe and $S_{\cD}^{\mathrm{ho}}$.

Direct-risk optimization takes 20 projected-SGD steps on $S_{\cD}^{\mathrm{opt}}$ at learning rate $10^{-4}$, then evaluates $\widetilde\theta_k$ once on $S_{\cD}^{\mathrm{ho}}$, which never selects an iterate. Along with the terminal projected-gradient norm, we report
\begin{equation}
    \widehat{\mathcal A}_{\cD}(\cU_k)
    =
    \widehat{\cD}_{S_{\cD}^{\mathrm{ho}}}(\widetilde\theta_k)
    -
    \widehat{\cD}_{S_{\cD}^{\mathrm{ho}}}
    (\widetilde\theta_{k_{\max}}),
    \qquad
    k_{\max}=24.
    \label{eq:llm-approximation-proxy}
\end{equation}
This finite-budget held-out proxy is zero at $k_{\max}$ by construction and differs from the population quantity $\mathcal A_{\cD}(\cU_k)$.

\subsection{LLM Cross-Fitting and Seed Aggregation}
\label{app:llm-crossfit}

For cross-fitting, we denote the fixed 256-example halves of each 512-example surrogate and decision probe by $A$ and $B$. Witness construction and selected-utility evaluation are run in both directions, using one fold for construction or selection and the other for finite-step evaluation, then averaged within seed.

\paragraph{Witness cross-fitting.}
At $\rho=\rho_{\max}=0.75$ and $k=8$, gradients on $(S_{\cL,A},S_{\cD,A})$ construct $v_\lambda$, whose finite-step $\Lambda(v)$ and $U(v)$ are evaluated on $(S_{\cL,B},S_{\cD,B})$. A seed passes only if both fold directions satisfy $\Lambda(v)>0>U(v)$. Table~\ref{tab:llm-calibration}(b) reports fold-averaged, seed-level witness effects at $\rho_{\max}$, whereas panel~(a) uses the full natural or counterfactual probes at $\theta_0$.

\paragraph{Selection cross-fitting.}
On fold $A$, learnability, fixed-reference gradient cosine, or random selection chooses a candidate, whose utility is evaluated on $S_{\cD,B}$. The decision oracle selects on $S_{\cD,A}$ and is likewise evaluated on $S_{\cD,B}$. Table~\ref{tab:llm-selection} reports fold-averaged, seed-level utilities and learnability--oracle gaps. Spearman correlation uses full fixed surrogate and decision probes. Random, gradient-cosine, and oracle utilities are $\rho$-invariant because candidate updates and decision probes are fixed.

\paragraph{Seed aggregation.}
All LLM summaries report medians over five paired warm-start seeds and 95\% percentile intervals from 8,000 seed-level bootstrap resamples. These intervals describe warm-start variation conditional on the fixed probes.

\end{document}